\documentclass[twoside,11pt]{article}

\usepackage{amsthm}
\usepackage[abbrvbib,preprint]{jmlr2e}
\usepackage[T1]{fontenc}
\usepackage{mathrsfs}
\usepackage{amsmath}
\usepackage{amssymb}
\usepackage{enumerate}
\usepackage{tikz}
\usetikzlibrary{patterns, arrows.meta, decorations.pathreplacing, calligraphy}

\definecolor{hanblue}{rgb}{0.27, 0.42, 0.81}
\definecolor{mordantred19}{rgb}{0.68, 0.05, 0.0}
\definecolor{red}{rgb}{0.68, 0.05, 0.0}
\definecolor{green}{rgb}{0.0, 0.5, 0.0}

\usepackage{color,soul}
\usepackage{hyperref}

\hypersetup{hidelinks}

\usepackage{aliascnt}
\let\oldtheorem\newtheorem
\RenewDocumentCommand{\newtheorem}{s m o m O{}}{%
	\IfBooleanTF{#1}%
	{\oldtheorem{#2}{#4}}%
	{\IfNoValueTF{#3}{\oldtheorem{#2}{#4}[#5]}%
		{\newaliascnt{#2}{#3}%
			\oldtheorem{#2}[#2]{#4}%
			\aliascntresetthe{#2}}}}

\newtheorem{example}{Example} 
\newtheorem{theorem}{Theorem}
\newtheorem{lemma}[theorem]{Lemma} 
\newtheorem{proposition}[theorem]{Proposition} 
\newtheorem{remark}[theorem]{Remark}

\newtheorem{definition}[theorem]{Definition}

\usepackage[
backend=biber,
citestyle=numeric-comp,
bibstyle=ieee,
sorting=nyt,
isbn=false,
url=false,
citecounter=true,
]{biblatex}
\AtEveryBibitem{\clearlist{language}}

\DeclareSourcemap{
	\maps[datatype=bibtex]{
		\map{
			\step[fieldsource=title,
			match=\regexp{\\\$},
			replace=\regexp{\$}]
			\step[fieldsource=title,
			match=\regexp{\{\\textasciicircum\}},
			replace=\regexp{\^}]
			\step[fieldsource=title,
			match={{\\textasciicircum}\\{p,q\\}},
			replace={\^{p,q}}]
			\step[fieldsource=title,
			match=\regexp{\\\_},
			replace=\regexp{\_}]
			\step[fieldsource=title,
			match=\regexp{\{\\textbackslash\}},
			replace=\regexp{\\}]
			\step[fieldsource=title,
			match=\regexp{\\\{\{R\}\\\}},
			replace=\regexp{\{R\}}]
		}
	}
}
\usepackage[colorinlistoftodos]{todonotes}

\makeatletter
\define@key{todonotes}{A}[]{\setkeys{todonotes}{author=Ahmed, color=yellow!40, inline}}
\define@key{todonotes}{As}[]{\setkeys{todonotes}{author=Ahmed, color=yellow!40}} % 's' for side note

\define@key{todonotes}{E}[]{\setkeys{todonotes}{author=Elena, color=green!40, inline}}
\define@key{todonotes}{Es}[]{\setkeys{todonotes}{author=Elena, color=green!40}} % 's' for side note
\define@key{todonotes}{M}[]{\setkeys{todonotes}{author=Marcello, color=orange!40, inline}}
\define@key{todonotes}{Ms}[]{\setkeys{todonotes}{author=Marcello, color=orange!40}} % 's' for side note
\makeatother
\usepackage{marginnote}

\AtEveryBibitem{\clearlist{language}}

\DeclareSourcemap{
	\maps[datatype=bibtex]{
		\map{
			\step[fieldsource=title,
			match=\regexp{\\\$},
			replace=\regexp{\$}]
			\step[fieldsource=title,
			match=\regexp{\{\\textasciicircum\}},
			replace=\regexp{\^}]
			\step[fieldsource=title,
			match=\regexp{\\\_},
			replace=\regexp{\_}]
			\step[fieldsource=title,
			match=\regexp{\{\\textbackslash\}},
			replace=\regexp{\\}]
			\step[fieldsource=title,
			match=\regexp{\\\{\{R\}\\\}},
			replace=\regexp{\{R\}}]
		}
	}
}

\usepackage[capitalize]{cleveref}
\usepackage{float}

\newcommand{\rd}{\mathbb{R}^{d}}
\newcommand{\nn}[1]{\mathbb{N}^{#1}}
\usepackage{xparse}

\NewDocumentCommand{\rr}{g}{%
	\mathbb{R}%
	\IfNoValueF{#1}{%
		\if\relax\detokenize{#1}\relax
		\else
		^{#1}%
		\fi
	}%
}

\newcommand{\zz}[1]{\mathbb{Z}^{#1}}

\newcommand{\zzp}[1]{\mathbb{Z}_+^{#1}}

\newcommand{\abs}[1]{| #1|}     %%%%%   for |x|
\newcommand{\bb}[1]{\mathbb{#1}}

\newcommand{\norm}[2]{\|#1\|_{#2}}

\newcommand{\chF}[1]{\chi_{#1}}

\numberwithin{equation}{section}

\author{\name Ahmed Abdeljawad\email ahmed.abdeljawad@oeaw.ac.at \\
	\addr Johann Radon Institute of Computational and Applied Mathematics (RICAM)\\
	\addr Austrian Academy of Sciences\\
	\addr Altenberger Straße 69, A-4040 Linz, Austria\\
	\AND
	\name Marcello Carioni  \email m.c.carioni@utwente.nl \\
	\addr Department of Applied Mathematics \\
	\addr University of Twente \\
	\addr 7500AE Enschede, The Netherlands\\
	\AND
	\name Elena Cordero  \email elena.cordero@unito.it \\
	\addr Dipartimento di Matematica\\
	\addr Università degli Studi di Torino\\
	\addr via Carlo Alberto 10, 10123 Torino, Italy
}
\editor{}

\begin{document}
	\title{Approximation Rates for Metaplectic Neural Networks}

	\maketitle
	\begin{abstract}
		In this paper we develop quantitative approximation results for shallow neural networks constructed using a dictionary based on metaplectic operators. First, we extend the concept of Barron spaces by considering a symplectically motivated extension of the Fourier transform, known as the metaplectic transform. Then, after establishing embedding between metaplectic Barron spaces and Sobolev spaces we consider a neural metaplectic dictionary and we prove Monte-Carlo approximation bounds for metaplectic Barron functions using finite linear combinations of atoms of the dictionary. Finally, we validate the introduction of the neural metaplectic dictionary by devising a deep neural network architecture that uses as building blocks the atoms of the dictionary. We test it to approximate solutions of time-dependent Schrödinger equations, demonstrating better performance compared to classical phyisics informed neural networks architectures.  
	\end{abstract}

	\begin{keywords}
		{Approximation Rate, Neural Network, Barron Space, Curse of Dimensionality}\\
		\noindent\textbf{2020 MSC:} 
		41A25, 41A46, 41A30, 41A65, 46E35, 68T07, 62M45, 68T05. 
	\end{keywords}
	% \vspace{.5cm}
	
	%%%%%%%%%%%%%%%%%%%%%%%%%%%%%%%%%%%%%%%%%%%%%%%%%%%%%%%%%%%%%%%%%%
	\section{Introduction}

	Neural networks provide flexible non-linear approximation models whose expressive power has led to their widespread use with applications including natural language processing, computer vision and scientific computing.    
	Their remarkable empirical success has, in turn, motivated mathematical research aimed at explaining why neural networks perform so well across such a broad range of tasks. This research has addressed several characteristic features of neural networks, including their high expressive power, their apparent ability to mitigate the curse of dimensionality, and their favorable behavior under commonly used optimization algorithms. Beyond improving our theoretical understanding, these developments have also contributed to the design of more effective architectures and training methods.
	
	A central question in the study of the expressive power of neural networks is to identify classes of functions that can be approximated efficiently by finite-width networks. Classical universal approximation theorems, beginning with the foundational works  \cite{Cybenko89ApproximationSuperpositionsSigmoidal, hornik1991approximation, hornik1989multilayer} and followed by several refinements \cite{lu2017expressive, Elbrachter2021DeepNeuralNetwork, de2021approximation, Abdeljawad2025UniformApproximationQuadratic, Abdeljawad22ApproximationsDeepNeural}, establish that sufficiently wide and deep neural networks can approximate broad classes of functions that can additionally arise as PDE solutions \cite{Kutyniok2022TheoreticalAnalysisDeep, Chen23RegularityTheoryStatic, kovachki2021universal, weinan2022some}, operator approximation \cite{chen1995universal, lanthaler2022error, kovachki2021universal} and in many other tasks where the approximation of functions with neural networks is central.

	Besides universal approximation results, quantitative approximation results, aim to determine explicit decay rates for the approximation error as a function of the network width. A foundational contribution in this direction was made by Barron \cite{Barron93UniversalApproximationBounds}, who introduced a Fourier-analytic class of functions, named Barron spaces, that can be approximated by shallow neural networks with an error of order $N^{-1/2}$, where $N$ denotes the number of neurons. A crucial feature of this estimate is that its exponent does not depend on the ambient dimension. Consequently, whenever the relevant Barron norm can be bounded independently, or only mildly dependently, on the dimension, this result provides a mathematical explanation for the ability of neural networks to mitigate, at least partially, the curse of dimensionality.
	
	Barron’s work has inspired a substantial body of research aimed at deriving dimension-independent approximation estimates in different Banach spaces and Sobolev norms \cite{Siegel23CharacterizationVariationSpaces, Siegel20ApproximationRatesNeural} and extending qualitative and quantitative approximation properties in refined classes \cite{Abdeljawad23SpaceTimeApproximationShallow,Abdeljawad24WeightedApproximationBarron, Voigtlaender22SamplingNumbersFourierAnalytic, BarronCohenDahmenDeVore2008, E2022BarronSpaceFlowInduced}. Such extensions are particularly relevant for the numerical approximation of partial differential equations. In this setting, it is often necessary to control not only the approximation error of the solution itself, but also that of its derivatives. Under suitable regularity assumptions on the activation function, quantitative approximation estimates in Sobolev spaces can therefore provide a theoretical foundation for the use of neural networks in scientific machine learning \cite{Raissi2019PhysicsinformedNeuralNetworks, Chen23RegularityTheoryStatic}.
	
	Quantitative approximations are based on classical works by Maurey \cite{Pisier1981Maurey} and Jones \cite{Jones1992Greedy}, studying the rates of non-linear approximations of functions using suitable dictionaries. In particular, for shallow neural networks defined as $f_{(\xi,b)}(x) =  \sum_i a_i \sigma(\xi_i \cdot x + b_i)$ the classical dictionary consists of single-neuron ridge functions, namely,
	\begin{align*}
		\mathbb{D}_{\rm NN} = \{x\mapsto \sigma(\xi\cdot x + b) : (\xi,b) \in \mathbb{R}^d \times \mathbb{R}\}.
	\end{align*}

	These observations naturally lead to the question of whether analogous dimension-independent approximation results can be established for more general dictionaries \cite{Abdeljawad25TimeFrequencyAnalysisNeural, Parhi2023ModulationSpacesCurse}. Replacing the classical ridge-function dictionary by dictionaries adapted to the structure of a given problem may yield efficient approximation guarantees for function classes that are not contained in a standard Barron space, or for which verifying Barron regularity is difficult. This issue is relevant not only from the perspective of abstract approximation theory, but also for applications. Solutions of high-dimensional PDEs, parametric and multiscale problems, inverse problems, and operator-learning tasks may possess geometric, compositional, or anisotropic structures that are not naturally captured by classical Barron spaces. Approximation results for more general dictionaries could therefore provide a principled justification for the design of problem-adapted neural network architectures and for their use in physics-informed neural networks and other scientific machine-learning methods.

	\subsection{Main contributions}
	
	In this work, we move beyond the classical Barron-space framework by considering a symplectically motivated extension of the Fourier transform, known as the metaplectic transform \cite{deGosson1997}. Metaplectic transforms arise from the metaplectic group ${\rm Mp}(d,\mathbb{R})$, which is a nontrivial double cover of the symplectic group ${\rm Sp}(d,\mathbb{R})$ \cite{deGosson1997}. They generalize the Fourier transform by incorporating the broader class of transformations induced by linear symplectic geometry.
	
	More precisely, given a free symplectic matrix $S \in {\rm Sp}(d, \mathbb{R})$ with blocks $A,B,C,D$ a metaplectic transform is defined as
	\begin{equation}\label{eq:introMetapFree}
		\widehat{S}f(x):=
		\frac1{\sqrt{|\det B|}}
		\int_{\rd}
		e^{2\pi i\, W_S(x,\xi)}
		\,f(\xi)\,d\xi
	\end{equation}
	where
	\begin{equation}\label{eq:introdefGeneratingFunction}
		W_S(x,\xi)
		=
		\frac{1}{2}\, x\cdot D B^{-1} x
		- (B^{-1}x)\cdot \xi
		+ \frac{1}{2}\, \xi\cdot B^{-1} A \xi
	\end{equation}
	is the generating function of $S$. Note that by choosing $S = J$, where $J$ denotes the standard symplectic matrix, one recovers the classical Fourier transform.
	This observation motivates the introduction of metaplectic Barron spaces, defined by extending the classical notion of Barron spaces through the replacement of the Fourier transform with a metaplectic transform:
	\begin{align*}
		\mathscr  B^S_s(\mathbb R^d)
		=
		\left\{
		f\in\mathscr S'(\mathbb R^d):
		\int (1+ |\xi|)^s |\widehat Sf(\xi)| d\xi < \infty
		\right\}.
	\end{align*}
	
	Our first contribution is to establish embeddings between metaplectic Barron spaces and Sobolev spaces. Although metaplectic Barron spaces are, in general, not closed under differentiation, we show that the metaplectic Barron norm of a derivative $\partial^\alpha f$ can be controlled by metaplectic Barron norms of suitable linear combinations of polynomials of $f$, with the Barron order increased by $|\alpha|$. This estimate subsequently yields bounds for the Sobolev norm $W^{n,r}$ in terms of metaplectic Barron norms of suitable polynomially weighted combinations of $f$ of degree at most $n$.
	
	Having established these bounds, we turn to the approximation of metaplectic Barron functions using suitably designed metaplectic dictionaries. More precisely, given a free symplectic matrix $S\in{\rm Sp}(d,\mathbb{R})$, we introduce the following neural metaplectic dictionary, whose elements are chirped ridge-type atoms:
	
	\begin{align}
		\bb{D}_S = \left\{ x\mapsto\sigma\left(
		\omega\cdot B^{-1}x+b
		\right)
		\cos\left(
		\pi x\cdot(B^{-1}A)x-\theta
		\right) \text{ such that }(\omega,b,\theta) \in
		\rd\times\rr\times[0,2\pi)\right\}
	\end{align}
	where $\sigma$ is a non-linear function. This dictionary extends the classical shallow neural network dictionary $\mathbb D_{\rm NN}$ and reduces to it in the case $S=J$ (see \cref{eq:symp} regarding the definition of $J$). Building on Maurey’s approximation theory, we establish Monte-Carlo approximation bounds in Sobolev norms for metaplectic Barron functions $f\in \mathscr B^S_{n+1}(\mathbb R^d)$ using finite combinations of atoms from $\mathbb D_S$. More precisely, on bounded domains and under polynomially decay and regularity assumptions on the activation function $\sigma$, we prove that
	\begin{equation}\label{eq:introMC_rate_Wmr_noeps}
		\inf_{f_N\in \Sigma_N(\bb{D}_S)}\|f-f_N\|_{W^{n,r}(\Omega)}
		\le
		C\,
		N^{-1/2}\|f\|_{ \mathscr B^S_{{{n+1}}}(\rd)},
	\end{equation}
	where $\Sigma_N(\bb{D}_S)$ is the set of linear combinations of elements of the dictionary $\bb{D}_S$. This result is then extended to unbounded domains.
	
	Finally, inspired by the structure of the dictionary $\mathbb D_S$, we design a \emph{deep} neural network architecture to approximate solutions to the time-dependent Schrödinger equation with a harmonic oscillator potential. 
	The choice of the Schrödinger equation as a benchmark is not incidental, but is
	strongly motivated by the intrinsic connection between Schrödinger evolutions and
	the metaplectic group. In fact, whenever the Hamiltonian is a real quadratic form,
	the corresponding Schrödinger propagator is exactly a metaplectic operator acting
	on the initial datum. More generally, Schrödinger equations with sufficiently regular
	perturbations can be represented by generalized metaplectic operators. This observation suggests that metaplectic atoms are
	naturally adapted to represent oscillatory wave functions generated by Schrödinger
	dynamics. Consequently, replacing the dictionary $\mathbb{D}_{\rm NN}$ with the
	metaplectic-inspired one $\mathbb D_S$ in the neural network layers is expected to produce more efficient approximations. To test this hypothesis, we train the resulting metaplectic-inspired model using a physics-informed neural network loss \cite{Raissi2019PhysicsinformedNeuralNetworks} and compare its performance with that of a standard deep neural network without metaplectic components. Our numerical experiments indicate that the additional flexibility provided by the metaplectic architecture leads to improved approximation accuracy. In particular, the performance advantage of the metaplectic neural network over the traditional architecture becomes increasingly pronounced as the number of eigenmodes in the reference solution grows.
	
	\subsection{Outline of the paper}
	
	The paper is organized as follows. In Section \ref{sec:preliminaries}, we review the necessary background on symplectic and metaplectic operators, together with variation spaces and Maurey-type approximation theory. In Section \ref{sec:funcspaces}, we introduce metaplectic Barron spaces and establish their main analytic and Sobolev regularity properties. In Section \ref{sec:space_approximation}, we define the neural metaplectic dictionary and derive quantitative approximation rates on bounded and unbounded domains. Finally, in Section \ref{sec:numerical_schrodinger_pinn}, we test the proposed architecture on a physics-informed neural-network benchmark for the time-dependent Schrödinger equation.
	
	%%%%%%%%%%%%%%%%%%%%%%%%%%%%%%%%%%%%%%%%%%%%%%%%%%%%%%%%%%%%%%%%%%

	%%%%%%%%%%%%%%%%%%%%%%%%%%%%%%%%%%%%%%%%%%%%%%%%%%%%%%%%%%%%%%%%%%
	\section{Preliminaries}\label{sec:preliminaries}
	%%%%%%%%%%%%%%%%%%%%%%%%%%%%%%%%%%%%%%%%%%%%%%%%%%%%%%%%%%%%%%%%%%
	This section gathers the preliminary material needed in the sequel. We recall basic facts on symplectic matrices and their associated metaplectic operators, and then introduce variation spaces generated by dictionaries of atoms. The latter framework allows us to formulate Maurey-type approximation estimates in a form suited to the sparse approximation results developed later.
	
	\subsection{The symplectic group}\label{sec:sympecticgroup}
	
	We begin by recalling the basic notation and structural properties of the symplectic group. For a thorough introduction to symplectic linear algebra and its relation to metaplectic analysis, we refer to \cite{degosson2011symplectic}.
	
	A matrix $S\in\rr{2d\times 2d}$ is said to be symplectic  if and only if
	\begin{equation}\label{eq:symp}
		S^T J S = J,
	\end{equation}
	where
	\begin{equation}\label{defJ}
		J=
		\begin{pmatrix}
			0_d & I_d\\
			- I_d & 0_d
		\end{pmatrix},
	\end{equation}
	$J$ is also called the standard symplectic matrix.
	We denote by $\mathrm{Sp}(d,\rr)$ the set of all real $2d\times 2d$ symplectic matrices. Throughout, we write such matrices in block form as
	\begin{equation}\label{blockS}
		S=
		\begin{pmatrix}
			A & B\\
			C & D
		\end{pmatrix},
	\end{equation}
	where $A,B,C,D\in\rr{d\times d}$. In terms of these blocks, condition \eqref{eq:symp} is equivalent to the compatibility relations
	\begin{equation}\label{sympRel}
		\begin{cases}
			A^T C = C^T A,\\
			B^T D = D^T B,\\
			A^T D - C^T B = I_d.
		\end{cases}
	\end{equation}
	Equivalently, if $S$ has the block form \eqref{blockS}, then $S$ is symplectic precisely when its inverse is given by
	\begin{equation}\label{blockS-1}
		S^{-1}=
		\begin{pmatrix}
			D^T & -B^T\\
			- C^T & A^T
		\end{pmatrix}.
	\end{equation}
	
	A symplectic matrix $S$ is called \emph{free} whenever the block $B$ is invertible, i.e.\ $B\in\mathrm{GL}(d,\rr)$. This class is particularly important because free symplectic matrices admit an explicit description in terms of quadratic generating functions and oscillatory integral operators.
	
	\begin{remark}\label{rem:symm}
		The relations \eqref{sympRel} may equivalently be written in the form
		\begin{equation}\label{sympRel-1}
			\begin{cases}
				C D^T = D C^T,\\
				A B^T = B A^T,\\
				A D^T - B C^T = I_d.
			\end{cases}
		\end{equation}
		In particular, when $B$ is invertible, the symmetry condition $A B^T = B A^T$ is equivalent to
		\begin{equation}\label{eq:ABtrasp}
			B^{-1}A = (B^{-1}A)^T,
		\end{equation}
		so that $B^{-1}A$ is symmetric. Similarly, the condition $B^T D= D^T B$ in \eqref{sympRel} is equivalent to
		\begin{equation}\label{eq:DBtrasp}
			DB^{-1} = (DB^{-1})^T,
		\end{equation}
		so that $DB^{-1}$ is symmetric.
	\end{remark}
	We write
	\[
	\mathrm{Sym}(d,\rr):=\{Q\in\rr{d\times d}: Q^T=Q\}
	\]
	for the space of real symmetric $d\times d$ matrices.
	
	A useful structural fact is that $\mathrm{Sp}(d,\rr)$ is generated by the matrix $J$ together with two elementary families of symplectic matrices. More precisely, see for instance \cite[Corollary 63]{degosson2011symplectic}, the symplectic group is generated by $J$ and by the matrices
	\begin{align}
		V_Q &=
		\begin{pmatrix}
			I_d & 0_d\\
			Q & I_d
		\end{pmatrix},
		\qquad Q\in\mathrm{Sym}(d,\rr), \label{defVP} \\
		\mathcal{D}_E &=
		\begin{pmatrix}
			E^{-1} & 0_d\\
			0_d & E^T
		\end{pmatrix},
		\qquad E\in\mathrm{GL}(d,\rr). \label{defDL}
	\end{align}
	
	\begin{example}
		We recall two elementary consequences of the preceding definitions.
		\begin{enumerate}
			\item[(a)] If $P\in\mathrm{Sym}(d,\rr)$, then the upper triangular matrix
			\begin{equation}\label{defUQ}
				U_P = V_P^T =
				\begin{pmatrix}
					I_d & P\\
					0_d & I_d
				\end{pmatrix}
			\end{equation}
			belongs to $\mathrm{Sp}(d,\rr)$.
			
			\item[(b)] The symplectic group is closed under transposition and inversion. Hence, if $S\in\mathrm{Sp}(d,\rr)$, then both $S^T$ and $S^{-1}$ are again symplectic.
		\end{enumerate}
	\end{example}

	\subsection{Metaplectic operators}\label{sec:metap-op}
	
	We now recall the basic link between symplectic matrices and metaplectic operators. The \emph{metaplectic group}, denoted by $\mathrm{Mp}(d,\rr)$, is the non-trivial double cover of the symplectic group $\mathrm{Sp}(d,\rr)$. Its elements are unitary operators on $L^2(\rr{d})$ and are called \emph{metaplectic operators}. We refer to \cite{degosson2011symplectic} for a detailed account.
	
	There exists a surjective group homomorphism
	\[
	\pi^{\mathrm{Mp}} : \mathrm{Mp}(d,\rr) \longrightarrow \mathrm{Sp}(d,\rr),
	\]
	whose kernel is
	\[
	\ker(\pi^{\mathrm{Mp}})=\{\pm I\}.
	\]
	Thus every symplectic matrix $S\in \mathrm{Sp}(d,\rr)$ has exactly two metaplectic preimages, denoted by $\pm \widehat S$ belong to $\mathrm{Mp}(d,\rr)$.
	
	\begin{definition}[Metaplectic operator]
		A metaplectic operator is an element $\widehat S\in \mathrm{Mp}(d,\rr)$. Its associated symplectic matrix is
		\[
		S=\pi^{\mathrm{Mp}}(\widehat S)\in \mathrm{Sp}(d,\rr).
		\]
	\end{definition}
	
	\begin{example}\label{exMetap}
		The following are standard examples of metaplectic operators together with their symplectic projections.
		\begin{enumerate}
			\item[(a)]
			The Fourier transform $\mathscr{F}$,
			\begin{equation*}
				\mathscr{F}f(\xi)
				=
				\widehat{f}(\xi)
				=
				\int_{\rd}
				f(x)e^{-2\pi i \xi\cdot x}\,dx,
			\end{equation*}
			is a metaplectic operator. Its projection is $\pi^{\mathrm{Mp}}(\mathscr{F})=J$.
			
			\item[(b)] For $Q\in\mathrm{Sym}(d,\rr)$, define the chirp
			\begin{equation}\label{defPhi}
				\Phi_Q(t)=e^{i\pi Q t\cdot t}.
			\end{equation}
			The multiplication operator
			\begin{equation}\label{eq:chirp}
				\mathfrak{p}_Q f(t)=\Phi_Q(t)f(t)
			\end{equation}
			is metaplectic and projects onto $V_Q$, defined in \eqref{defVP}.
			
			\item[(c)] If $E\in\mathrm{GL}(d,\rr)$, the dilation operator
			\begin{equation}\label{eq:dil}
				\mathfrak{T}_E f(t)
				=
				|\det E|^{1/2} f(Et)
			\end{equation}
			is metaplectic and projects onto $\mathcal{D}_E$, defined in \eqref{defDL}.
			
			\item[(d)] For $P\in\mathrm{Sym}(d,\rr)$, the operator
			\begin{equation}\label{eq:mult}
				\mathfrak{m}_P f
				=
				\mathscr{F}^{-1}\big(\Phi_{-P}\widehat{f}\big)
			\end{equation}
			is metaplectic and projects onto $U_P$, defined in \eqref{defUQ}.
		\end{enumerate}
	\end{example}

	\subsection{Free symplectic matrices and related metaplectic operators}
	
	We now focus on free symplectic matrices, that is, to matrices $S\in\mathrm{Sp}(d,\rr)$ whose block $B$ in \eqref{blockS} is invertible. In this case, by Remark \ref{rem:symm}, both $DB^{-1}$ and $B^{-1}A$ are symmetric. This makes it possible to write $S$ in terms of the elementary symplectic matrices introduced above. One convenient factorization is
	\begin{equation}
		\label{eq:free-factorization-1}
		S
		=
		V_{D B^{-1}} \, \mathcal{D}_{B^{-1}} \, J \, V_{B^{-1}A}.
	\end{equation}
	Equivalently, using the upper triangular matrices in \eqref{defUQ}, one may write
	\begin{equation}
		\label{eq:free-factorization-2}
		S
		=
		V _{D B^{-1}} \mathcal{D}_{B^{-1}}  V^{T}_{-B^{-1}A}  J.
	\end{equation}
	At the metaplectic level, these factorizations read, up to a unimodular phase factor,
	\begin{align}
		\widehat{S}
		&=
		\mathfrak{p}_{D B^{-1}} \, \mathfrak{T}_{B^{-1}} \, \mathscr{F} \, \mathfrak{p}_{B^{-1}A},
		\label{factorization1}\\
		&=
		\mathfrak{p}_{D B^{-1}} \, \mathfrak{T}_{B^{-1}} \,\mathfrak{m}_{-B^{-1}A} \mathscr{F}.
		\label{factorization2}
	\end{align}
	
	Free symplectic matrices also admit a description by quadratic generating functions. To a free symplectic matrix $S$ we associate the quadratic function
	\begin{equation}\label{defGeneratingFunction}
		W_S(x,\xi)
		=
		\frac{1}{2}\, x\cdot D B^{-1} x
		- (B^{-1}x)\cdot \xi
		+ \frac{1}{2}\, \xi\cdot B^{-1} A \xi.
	\end{equation}
	The terminology ``generating function'' reflects the fact that $W_S$ encodes the corresponding symplectic transformation. Indeed, the linear transformation
	\begin{equation*}
		\begin{pmatrix}
			x\\
			\eta
		\end{pmatrix}
		=
		\begin{pmatrix}
			A & B\\
			C & D
		\end{pmatrix}
		\begin{pmatrix}
			\xi\\
			\zeta
		\end{pmatrix}
	\end{equation*}
	is equivalent to the Hamiltonian relations
	\begin{equation*}
		\begin{cases}
			\eta = \nabla_x W_S(x,\xi),\\
			\zeta = - \nabla_{\xi} W_S(x,\xi).
		\end{cases}
	\end{equation*}
	
	To each free symplectic matrix $S$ one can associate a pair of oscillatory integral operators $\widehat{S}_{W,m}$, defined for $f\in\mathscr{S}(\rr{d})$ by \cite{deGosson1997,deGosson2006}
	\begin{equation}\label{defMetapFreem}
		\widehat{S}_{W,m} f(x)
		=\frac{i^{\,m-\frac{d}{2}}}{
			\sqrt{|\det B|}}
		\int_{\rd}
		e^{2\pi i\, W_S(x,\xi)}
		\,f(\xi)\,d\xi.
	\end{equation}
	The integer parameter $m$ depends on the sign of $\det B$ and is chosen so that
	\begin{equation}
		m \equiv 0 \pmod{2}
		\quad \text{if } \det B > 0,
		\qquad
		m \equiv 1 \pmod{2}
		\quad \text{if } \det B < 0.
	\end{equation}
	It is classical that these oscillatory integral operators extend to unitary operators on $L^2(\rd)$, see, for instance, \cite{degosson2011symplectic}. Since in this paper the global phase is irrelevant, we shall henceforth work with the phase-free representative
	\begin{equation}\label{eq:defMetapFree}
		\widehat{S}f(x):=	\widehat{S}_{W} f(x)
		=
		\frac1{\sqrt{|\det B|}}
		\int_{\rd}
		e^{2\pi i\, W_S(x,\xi)}
		\,f(\xi)\,d\xi.
	\end{equation}
	
	The phase-free representative in \eqref{eq:defMetapFree} is the unitary
	operator given by the factorization \eqref{factorization1}. Inverting that
	factorization shows that its inverse is exactly the phase-free oscillatory
	representative associated with $S^{-1}$, whose block form is
	\eqref{blockS-1}. Since $B$ is invertible, its quadratic generating
	function is
	% \begin{equation*}
	% 	\widehat{S}^{-1} \;=\; \widehat{S^{-1}}
	% 	\qquad \text{up to a unimodular constant},
	% \end{equation*}
	% where $S^{-1}$ has the block form \eqref{blockS-1}. Since $B$ is invertible, the quadratic generating function associated with $S^{-1}$ is
	\begin{align}\label{eq:W-1}
		W_{S^{-1}}(x,\xi)
		=
		-\frac12\, x\cdot A^{T}B^{-T}x
		\;+\;
		(B^{-1}x)\cdot \xi
		\;-\;
		\frac12\, \xi\cdot B^{-T}D^{T}\xi.
	\end{align}
	Consequently,
	\begin{equation}\label{eq:inverse-metaplectic}
		\widehat{S}^{-1} f(x)
		=\frac1{
			\sqrt{|\det B|}}
		\int_{\rr{d}}
		e^{2\pi i\, W_{S^{-1}}(x,\xi)}
		\, f(\xi)\,d\xi.
	\end{equation}

	\subsection{Examples of free metaplectic transforms}
	
	We conclude this preliminary discussion with several important examples of free metaplectic transforms. These examples belong to the class of metaplectic operators, also known in optics and signal processing as linear canonical transforms, and arise naturally in harmonic analysis, phase-space analysis, and signal processing, see, for instance, \cite{folland1989phase, Namias1980, deGosson2006, degosson2011symplectic,Ozaktas2001}.
	
	\paragraph{Fractional Fourier transform.}
	The Fourier transform may be interpreted as a rotation of the time-frequency representation of a signal by an angle of $90^\circ$, mapping the time axis onto the frequency axis, see \cite{DiasGossonPrata2024}. More general rotations can be realized in optical systems and are described mathematically by the fractional Fourier transform (FRFT) \cite{Ozaktas1993, Almeida1994, Namias1980, Candan2000}. The FRFT belongs to the broader class of metaplectic operators, also known as linear canonical transforms.
	% , which play a central role in optics and signal processing \cite{Ozaktas2001}.
	
	Choose in \eqref{blockS}
	\[
	A = D = \mathrm{diag}(\cos\theta_1, \dots, \cos\theta_d),
	\qquad
	B = -C = \mathrm{diag}(\sin\theta_1, \dots, \sin\theta_d),
	\]
	with $\sin\theta_j\neq 0$. Then the associated metaplectic transform is the fractional Fourier transform. With the phase convention adopted in \eqref{eq:defMetapFree}, it is given by
	\[
	{\widehat S_\theta}f(\xi)
	=
	\frac{1}{\sqrt{\prod_{j=1}^d |\sin\theta_j|}}
	\int_{\rd}
	f(x)\,
	\exp\!\left(
	i\pi \sum_{k=1}^d
	\frac{(x_k^2+\xi_k^2)\cos\theta_k - 2x_k\xi_k}{\sin\theta_k}
	\right)\,dx.
	\]
	
	\paragraph{Fresnel transform.}
	If
	\[
	A = D = I_d,
	\qquad
	B = \mathrm{diag}(b_1,\dots,b_d), \quad b_j \neq 0,
	\qquad
	C = 0,
	\]
	then we obtain the Fresnel transform
	\[
	{\widehat S_b}f(\xi)
	=
	\frac{1}{\sqrt{\prod_{j=1}^d |b_j|}}
	\int_{\rd}
	f(x)\,
	\exp\!\left(
	i\pi \sum_{k=1}^d \frac{(x_k - \xi_k)^2}{b_k}
	\right)\,dx.
	\]
	
	\paragraph{Lorentz transform.}
	If
	\[
	A = D = \mathrm{diag}(\cosh\varphi_1, \dots, \cosh\varphi_d),
	\qquad
	B = C = \mathrm{diag}(\sinh\varphi_1, \dots, \sinh\varphi_d),
	\]
	with $\sinh\varphi_j\neq 0$, then the corresponding transform is
	\[
	{\widehat S_\varphi}f(\xi)
	=
	\frac{1}{\sqrt{\prod_{j=1}^d |\sinh\varphi_j|}}
	\int_{\rd}
	f(x)\,
	\exp\!\left(
	i\pi \sum_{k=1}^d
	\frac{(x_k^2+\xi_k^2)\cosh\varphi_k - 2x_k\xi_k}{\sinh\varphi_k}
	\right)\,dx.
	\]

	%%%%%%%%%%%%%%%%%%%%%%%%%%%%%%%%%%%%%%%%%%%%%%%%%%%%%%%%%%%%%%%%%%
	\subsection{Variation Spaces and Maurey-Type Approximation}
	\label{sec:prelim:VariationSpace}
	%%%%%%%%%%%%%%%%%%%%%%%%%%%%%%%%%%%%%%%%%%%%%%%%%%%%%%%%%%%%%%%%%%
	
	We recall the notion of variation spaces associated with a dictionary of atoms. This framework is closely related to atomic Banach spaces, non-linear approximation with dictionaries, and greedy approximation, see, for instance, \cite{DeVore98NonlinearApproximation,Temlyakov2011GreedyApproximation}. In the context of neural-network approximation, such spaces provide an abstract formulation of Barron-type and convex neural-network norms, see \cite{Barron93UniversalApproximationBounds,Bach17BreakingCurseDimensionality,Siegel24SharpBoundsApproximation,Siegel23CharacterizationVariationSpaces}.
	
	Let $\mathcal{B}$ be a Banach space, and let $\bb{D}\subset \mathcal{B}$ be a collection of elements, called a \emph{dictionary} or a family of \emph{atoms}. Any zero atom may be discarded without changing the generated approximation classes. In non-linear approximation, the order of the atoms is irrelevant, what matters is the number of selected atoms and the size of the corresponding coefficients.
	
	For $N\in \nn{}$ and $M>0$, we define the class of $N$-term approximants with $\ell^1$-controlled coefficients by
	\begin{equation*}
		\Sigma_{N,M}(\bb{D})
		:=
		\Bigl\{
		\textstyle\sum_{j=1}^{N} a_j h_j
		\;:\;
		h_j\in \bb{D},\;
		a_j\in \rr{},
		\sum_{j=1}^{N}|a_j|\le M
		\Bigr\}.
	\end{equation*}
	If the $\ell^1$-constraint is removed, one obtains the usual non-linear manifold of $N$-term expansions,
	\begin{equation*}
		\Sigma_N(\bb{D})
		:=
		\bigcup_{M>0}\Sigma_{N,M}(\bb{D}).
	\end{equation*}
	Thus, $\Sigma_{N,M}(\bb{D})$ consists of all linear combinations of at most $N$ atoms from $\bb{D}$ whose coefficients have total $\ell^1$-mass at most $M$.
	In this way, \(\Sigma_N(\bb D)\) can be identified with the class of
	shallow neural networks of width \(N\) whose activation is encoded by the
	dictionary \(\bb D\).
	
	Allowing arbitrary finite widths and then taking the closure gives precisely
	the closed symmetric convex hull of the dictionary:
	\[
	\overline{\operatorname{conv}(\pm\bb D)}
	=
	\overline{
		\bigcup_{N\in\nn{}}
		\Sigma_{N,1}(\bb D)
	}.
	\]
	This viewpoint naturally leads to the corresponding infinite-width model,
	namely the variation space associated with \(\bb D\).
	
	\begin{definition}\label{def:variationnorm}
		Let $\mathcal{B}$ be a Banach space and let $\bb{D}\subseteq \mathcal{B}$ be a dictionary.
		For $f\in\mathcal{B}$, the variation norm of $f$ with respect to $\bb{D}$ is defined by
		\begin{align*}
			\norm{f}{\mathcal{K}(\bb{D})}
			&:=
			\inf\bigl\{c>0:\, f/c\in \overline{\operatorname{conv}(\pm\bb{D})}\bigr\},
			\intertext{where the closure is taken in $\mathcal{B}$. The associated variation space is}
			\mathcal{K}(\bb{D})
			&:=
			\bigl\{f\in\mathcal{B}:\norm{f}{\mathcal{K}(\bb{D})}<\infty\bigr\}.
		\end{align*}
	\end{definition}
	
	The relevance of the variation norm is that it directly controls sparse approximation rates. The mechanism behind this fact is Maurey's empirical method, originally due to Maurey and developed in the Banach-space setting in \cite{Pisier1981Maurey}. In Hilbert spaces and learning-theoretic settings, closely related estimates appear in \cite{Jones1992Greedy,Barron93UniversalApproximationBounds,BarronCohenDahmenDeVore2008}. The following form is the version needed here, it is an adaptation of the Maurey-type approximation estimate for variation spaces in \cite{Siegel24SharpBoundsApproximation}.
	\begin{proposition}[Approximation Rate in Type-2 Banach Spaces]
		\label{prop:approximation_type2}
		Let $\mathcal{B}$ be a type-2 Banach space, and let $\bb{D}\subset \mathcal{B}$ be a dictionary satisfying
		\begin{align}  \label{eq:boudKd} 		K_{\bb{D}}:=\sup_{d\in\bb{D}}\|d\|_{\mathcal{B}}<\infty.
		\end{align}
		Then, for every $f\in \mathcal{K}(\bb{D})$ and every $N\in\nn{}$, one has
		\begin{align*}
			\inf_{f_N\in \Sigma_{N,M_f}(\bb{D})}\|f-f_N\|_{\mathcal{B}}
			\le
			4 C_{2,\mathcal{B}} K_{\bb{D}} M_{f} N^{-1/2},
		\end{align*}
		where $M_f=\norm{f}{\mathcal{K}(\bb{D})}$. Here $C_{2,\mathcal{B}}$ denotes a type-2 constant of $\mathcal{B}$, that is, a constant for which
		\[
		\left(\mathbb{E}\left\|\sum_{j=1}^m \varepsilon_j x_j\right\|_{\mathcal{B}}^2\right)^{1/2}
		\le
		C_{2,\mathcal{B}}
		\left(\sum_{j=1}^m \|x_j\|_{\mathcal{B}}^2\right)^{1/2}
		\]
		for every finite families $x_1,\dots,x_m\in\mathcal{B}$. Here $\varepsilon_1,\dots,\varepsilon_m$ are independent Rademacher random variables.
	\end{proposition}
	
	The estimate above shows that elements of $\mathcal{K}(\bb{D})$ admit dimension-free sparse approximations at the rate $N^{-1/2}$, with constants depending only on the type-2 geometry of the ambient Banach space, the uniform size of the atoms, and the variation norm of the target. Potential extensions and related forms of Maurey's approximation principle are discussed in \cite[Section 8]{DeVore98NonlinearApproximation}.
	
	For later use, we also recall a measure-theoretic characterization of variation spaces. General background on total variation, image
	measures, and vector-valued integration may be found in
	\cite{BogachevMeasureTheory}.
    This representation is particularly useful when the elements of $\mathcal{K}(\bb{D})$ are interpreted as infinite-width superpositions of atoms.

	\begin{proposition}[{\cite[Lemma 3]{Siegel23CharacterizationVariationSpaces}}]
		\label{prop:variation_norm}
		Let $\mathcal{B}$ be a Banach space and suppose that $\bb{D}\subset\mathcal{B}$ is bounded.
		Then $f\in\mathcal{K}(\bb{D})$ if and only if there exists a Borel measure $\mu$ on $\bb{D}$ such that
		\begin{align*}
			f=\int_{\bb{D}}i_{\bb{D}\to\mathcal{B}}\,d\mu,
		\end{align*}
		where $i_{\bb{D}\to\mathcal{B}}$ denotes the canonical inclusion. Moreover,
		\begin{align*}
			\norm{f}{\mathcal{K}(\bb{D})}
			=
			\inf\left\{\norm{\mu}{}:
			f=\int_{\bb{D}}i_{\bb{D}\to\mathcal{B}}\,d\mu\right\},
		\end{align*}
		where the infimum is taken over all Borel measures $\mu$
		on $\bb{D}$, and $\|\mu\|$ denotes the total variation norm of $\mu$.
	\end{proposition}

	%%%%%%%%%%%%%%%%%%%%%%%%%%%%%%%%%%%%%%%%%%%%%%%%%%%%%%%%%%%%%%%%%%
	\section{Function Spaces for Metaplectic Approximation}\label{sec:funcspaces}
	For $s\in\rr$, we define
	\begin{equation}\label{eq:vs}
		v_s(\xi)=(1+|\xi|)^{s}.
	\end{equation}
	For $s\geq0$,  $v_s$  is \textit{submultiplicative} \cite[Lemma 11.1.1]{Grochenig01FoundationsTimeFrequencyAnalysis}:
	$$ v_s(\xi+\eta)\leq v_s(\xi)v_s(\eta),\quad \xi,\eta\in\rd.$$
	Furthermore, for every $s\in\rr$ they are equivalent to the weights
	\begin{equation}\label{eq:lrs}
		\langle\xi\rangle^s:=(1+|\xi|^2)^{s/2},\quad s\in\rr.
	\end{equation}
	We will use these equivalent expressions interchangeably without further mention. 
	\begin{definition}[Sobolev spaces]
		Let $\Omega\subseteq\mathbb R^d$ be an open set, $k\in\mathbb{Z}_+$, and
		$1\le p\le\infty$. The Sobolev space $W^{k,p}(\Omega)$ consists of all
		functions $f\in L^p(\Omega)$ whose weak derivatives satisfy
		\[
		\partial^\alpha f\in L^p(\Omega),
		\qquad |\alpha|\le k.
		\]
		It is endowed with the norm
		\[
		\|f\|_{W^{k,p}(\Omega)}
		=
		\sum_{|\alpha|\le k}
		\|\partial^\alpha f\|_{L^p(\Omega)}.
		\]
	\end{definition}
	
	\begin{definition}[Weighted Sobolev spaces]
		Let $k\in\mathbb{Z}_+$, $1\le p\le\infty$, and let
		$v:\mathbb R^d\to(0,\infty)$ be a weight. The weighted Sobolev space
		$W^{k,p}(v;\mathbb R^d)$ consists of all functions
		$f\in L^p_{\mathrm{loc}}(\mathbb R^d)$ such that
		\[
		v\,\partial^\alpha f\in L^p(\mathbb R^d),
		\qquad |\alpha|\le k.
		\]
		It is equipped with the norm
		\[
		\|f\|_{W^{k,p}(v)}
		=
		\sum_{|\alpha|\le k}
		\|v\,\partial^\alpha f\|_{L^p}.
		\]
		For the polynomial weight $v_s$ in \eqref{eq:vs}
		we simply write
		\[
		W^{k,p}(v_s;\mathbb R^d).
		\]
	\end{definition}
	%%%%%%%%%%%%%%%%%%%%%%%%%%%%%%%%%%%%%%%%%%%%%%%%%%%%%%%%%%%%%%%%%%
	\subsection{Metaplectic Barron Spaces}
	\label{subsec:ConvergenceInBochnerSobolevNorms}
	Barron spaces originate from the pioneering work of Barron \cite{Barron93UniversalApproximationBounds} and have subsequently been investigated and extended in several directions, see for instance 
	\cite{E2022BarronSpaceFlowInduced, Voigtlaender22SamplingNumbersFourierAnalytic, Abdeljawad23SpaceTimeApproximationShallow}. 
	They offer a Fourier-based characterization of functions that admit efficient approximation by shallow neural networks.
	
	\begin{definition}[Barron norm and Barron space]
		Let $s \in \rr{}$. The \emph{Barron space} of order $s$ on $\rr{d}$ is defined by
		\begin{equation*}
			\mathscr B_s(\rr{d})
			=
			\left\{
			f \in \mathscr{S}'(\rr{d})
			:\;
			\|f\|_{\mathscr B_s} < \infty
			\right\},
		\end{equation*}
		where the associated \emph{Barron norm} is given by
		\begin{equation}\label{eq:FLs}
			\|f\|_{\mathscr B_s}
			=
			\int_{\rr{d}}
			v_s(\xi) \, |\widehat{f}(\xi)|
			\, d\xi .
		\end{equation}
	\end{definition}
    
	For $1\le q\le\infty$ and $t\in\rr$, we write
	$$
	L^q_{v_t}(\rd)
	:=
	\{g\in\mathscr S'(\rd):v_tg\in L^q(\rd)\},
	\qquad
	\|g\|_{L^q_{v_t}}
	:=
	\|v_tg\|_{L^q},
	$$
	and define the Fourier--Lebesgue space
	$$
	\mathscr FL^q_{v_t}(\rd)
	:=
	\{f\in\mathscr S'(\rd):\widehat f\in L^q_{v_t}(\rd)\},
	\qquad
	\|f\|_{\mathscr FL^q_{v_t}}
	:=
	\|\widehat f\|_{L^q_{v_t}}.
	$$
	We use $\mathscr FL^q:=\mathscr FL^q_{v_0}$. Since $v_s$ and
	$\langle\cdot\rangle^s$ are equivalent weights, the Barron and
	Fourier--Lebesgue spaces coincide as sets with equivalent norms:
	\begin{equation}
		\label{eq:Barron-identity}
		\mathscr B_s(\rd)
		=
		\mathscr FL^1_{v_s}(\rd)
		=
		\mathscr FL^1_{\langle\cdot\rangle^s}(\rd),
		\qquad s\in\rr.
	\end{equation}
	See, for example,
	\cite{Cordero20TimeFrequencyAnalysisOperators,Grochenig01FoundationsTimeFrequencyAnalysis}.
	
	\noindent	
	{\bf Metaplectic Barron spaces.}
	We now introduce a metaplectic extension of the classical Barron spaces. 
	
	\begin{definition}[Metaplectic Barron space]\label{def:metabarron}
		Let $S\in\mathrm Sp(d,\mathbb R)$ and let $\widehat S$ be an associated metaplectic operator.
		For $s\in\mathbb R$, define
		\[
		\mathscr  B^S_s(\mathbb R^d)
		=
		\left\{
		f\in\mathscr S'(\mathbb R^d):
		\widehat Sf\in L^1_{v_s}(\mathbb R^d)
		\right\},
		\]
		with norm
		\[
		\|f\|_{B^S_s}
		=
		\int_{\mathbb R^d}
		v_s(\xi)
		|\widehat Sf(\xi)|
		\,d\xi .
		\]
		Equivalently,
		\[
		\mathscr  B^S_s
		=
		\widehat S^{-1}
		\big(
		L^1_{v_s}
		\big).
		\]
		If $S=J$, then
		\[
		\mathscr  B^J_s=\mathscr  B_s.
		\]
	\end{definition}
	
	Given a free symplectic matrix $S$ with block decomposition \eqref{blockS}, (i.e. $B\in\mathrm{GL}(d,\rr)$), we define the weighted dilation constant
	\begin{equation}\label{Dsq}
		D_{s}(B^{-1}):=       
		\sup_{\xi\in\rd}
		\frac{v_s(\xi)}{v_s( B^{T}\xi)}
		\le
		\max\{1,\|B^T\|,\|B^{-T}\|\}^{|s|}
		<\infty,\quad s\in\rr{} .
	\end{equation}
	Observe that for $s=0$ we have $D_{0}(B^{-1})=1$.
	
	Note that the following boundedness estimate for the metaplectic operator  $\mathfrak{T}_{B^{-1}}$ holds:
	\begin{equation}\label{eq:dilq}
		\|\mathfrak{T}_{B^{-1}}g\|_{L^{q}_{v_s}}
		\le
		|\det B|^{\frac{1}{q}-\frac{1}{2}}\,D_{s}(B^{-1})\,\|g\|_{L^{q}_{v_s}},
	\end{equation}
	for every $g\in L^{q}_{v_s}(\rr{})$, $0< q\leq\infty$. This estimate can be easily derived from the definition of $\mathfrak{T}_{B^{-1}}$ in \eqref{eq:dil}.
	Moreover, let us recall 	Young's weighted inequality. 
	\begin{lemma}[Weighted Young inequality] Consider $t,q,r\geq 1$ such that \[
		\frac{1}{t}+\frac{1}{q}=1+\frac{1}{r}.
		\]
		For every 
		$f\in L^t_{v_{|s|}}$, $g\in L^q_{v_s}$ with $s\in\rr$, we have
		\begin{equation}\label{eqYI}
			\|f*g\|_{ L^r_{v_s}}
			\le
			\|f\|_{L^t_{v_{|s|}}}
			\|g\|_{ L^q_{v_s}}.
		\end{equation}
	\end{lemma}
	
	The following localization estimate will be used in the Sobolev control below.
	
	\begin{lemma}\label{lem:M-localization}
		Fix $s\in\rr$ and assume that $\chi_\Omega\in \mathscr F L^r_{v_{|s|}}(\rd)$, with $r\geq1$.  Given a free symplectic matrix $S$ with block decomposition \eqref{blockS},
		for every $u\in\mathscr{S}(\rd)$, one has
		\begin{equation}\label{eq:M-localization}
			\|\widehat S(\chi_\Omega u)\|_{L^{r}_{v_{s}}}
			\le\, D_s(B^{-1})D_s(B)
			\,|\det B|^{\frac1r}
			\|\chi_\Omega\|_{\mathscr F L^{r}_{v_{|s|}}}\,
			\|u\|_{\mathscr B^S_{s}}.
		\end{equation}
	\end{lemma}
	
	\begin{proof} Assume that the free factorization \eqref{factorization1} holds with the following notation:
		\begin{equation}\label{eq:fact1}
			\widehat S = \mathfrak{p}_{Q_2}\, \mathfrak{T}_{B^{-1}}\, \mathscr{F}\, \mathfrak{p}_{Q_1},
			\qquad Q_1=-B^{-1}A,\qquad Q_2=DB^{-1}.
		\end{equation}
		Fix $u\in\mathscr{S}(\rd)$.
		Since $|\mathfrak{p}_{Q_2}(x)|=1$ pointwise, it is an isometry on $L^r_{v_s}$ and hence
		\begin{equation*}
			\|\widehat S(\chi_\Omega u)\|_{L^{r}_{v_s}}
			=
			\|\mathfrak{T}_{B^{-1}}\,\mathscr{F}\big( \mathfrak{p}_{Q_1}(\chi_\Omega u)\big)\|_{L^{r}_{v_s}}.
		\end{equation*}
		Apply the weighted dilation bound to get
		\begin{equation}\label{eq:M-step1}
			\|\widehat S(\chi_\Omega u)\|_{L^{r}_{v_s}}
			\le
			|\det B|^{\frac1r-\frac{1}{2}}\,D_{s}(B^{-1})\,
			\|\mathscr{F}\big(\mathfrak{p}_{Q_1}(\chi_\Omega u)\big)\|_{L^{r}_{v_s}}.
		\end{equation}
		Since $\mathfrak{p}_{Q_1}$ is a pointwise multiplier, it commutes with multiplication by $\chi_\Omega$:
		\begin{equation*}
			\mathfrak{p}_{Q_1}(\chi_\Omega u)=\chi_\Omega\,(\mathfrak{p}_{Q_1}u).
		\end{equation*}
		Therefore,
		\begin{equation*}
			\mathscr{F}\big(\mathfrak{p}_{Q_1}(\chi_\Omega u)\big)
			=
			\mathscr{F}(\chi_\Omega)\,*\,\mathscr{F}(\mathfrak{p}_{Q_1}u)
			=
			\widehat{\chi_\Omega}\,*\,\mathscr{F}(\mathfrak{p}_{Q_1}u).
		\end{equation*}
		By the weighted Young inequality \eqref{eqYI},
		\begin{equation}\label{eq:M-step2}
			\|\mathscr{F}\big(\mathfrak{p}_{Q_1}(\chi_\Omega u)\big)\|_{L^{r}_{v_s}}
			\le
			\,\|\widehat{\chi_\Omega}\|_{L^{r}_{v_{|s|}}}\,
			\|\mathscr{F}(\mathfrak{p}_{Q_1}u)\|_{L^{1}_{v_s}}.
		\end{equation}
		Finally, using
		\begin{equation*}
			\widehat Su=\mathfrak{p}_{Q_2}\,\mathfrak{T}_{B^{-1}}\,\mathscr{F}(\mathfrak{p}_{Q_1}u),
			\quad\text{so that}\quad
			\mathscr{F}(\mathfrak{p}_{Q_1}u)=\mathfrak{T}_B\,\mathfrak{p}_{-Q_2}\,\widehat Su,
		\end{equation*}
		and applying again the dilation estimate,
		\begin{equation*}
			\|\mathscr{F}(\mathfrak{p}_{Q_1}u)\|_{L^{1}_{v_s}}
			\le
			D_s(B)\,|\det B|^{-\frac12}\,\|\widehat Su\|_{L^{1}_{v_s}}.
		\end{equation*}
		Combining \eqref{eq:M-step1}-\eqref{eq:M-step2}, we obtain
		\begin{align*}
			\|\widehat S(\chi_\Omega u)\|_{L^{r}_{v_s}}
			&\le\, D_s(B^{-1})D_s(B)  |\det B|^{\frac1r}			\,\|\widehat{\chi_\Omega}\|_{L^r_{v_{|s|}}}\,
			\|\widehat Su\|_{L^1_{v_s}}\\
			&=\,D_s(B^{-1})D_s(B)\,|\det B|^{\frac1r}\,\,\|\chi_\Omega\|_{\mathscr{F}L^r_{v_{|s|}}}\,
			\|u\|_{ \mathscr B^S_{s}}.
		\end{align*}
		This proves \eqref{eq:M-localization}.
	\end{proof}
	The metaplectic Barron spaces \textit{are not closed under derivation}, as detailed below.
	\begin{proposition}\label{lem:derivative-control}
		Let   $\alpha\in\zzp{d}$ such that $|\alpha|>0$.  Then,  for every $s\in\rr{}$, $f\in\mathscr{S}(\rd)$, one has
		$$\|\partial^\alpha f\|_{\mathscr B^{S}_{s}}\leq 2^{|\alpha|}(2\pi)^{2|\alpha|} 
		|\det B|^{\frac{1}{2}}\,D_{s}(B^{-1})\sum_{\beta\leq \alpha}\|P_{\alpha-\beta}f\|_{\mathscr B^S_{s+|\beta|}},$$
		with $P_{\alpha-\beta}$ polynomial of degree $|\alpha-\beta|$, having coefficients depending on $\mathcal{Q}_1=-B^{-1}A$.   
	\end{proposition}
	\begin{proof}
		Fix $f\in\mathscr S(\rr{d})$ and $\alpha\in\zzp{d}$. By definition,
		\begin{equation*}
			\|\partial^\alpha f\|_{\mathscr B^S_{s}}
			=
			\|\widehat S(\partial^\alpha f)\|_{L^{1}_{v_s}}.
		\end{equation*}
		Since $|\mathfrak p_{Q_2}|=1$ pointwise, $\mathfrak p_{Q_2}$ is an isometry on every $L^{1}_{v_s}$, hence
		\begin{equation}\label{eq:D-step0}
			\|\widehat S(\partial^\alpha f)\|_{L^{1}_{v_s}}
			=
			\|\mathfrak{T}_{B^{-1}}\mathscr{F}(\mathfrak p_{Q_1}\partial^\alpha f)\|_{L^{1}_{v_s}}.
		\end{equation}
		Since $\mathfrak p_{Q_1}\partial^\alpha f = e^{\pi i y\cdot Q_1 y}\,\partial^\alpha f$, by repeated integration by parts,
		\begin{align}
			\mathscr{F}(\mathfrak p_{Q_1}\partial^\alpha f)(\xi)
			&=\int_{\rr{d}} e^{-2\pi i\,\xi\cdot y}\,e^{\pi i\,y\cdot Q_1y}\,\partial^\alpha f(y)\,dy \notag\\
			&=(-1)^{|\alpha|}\int_{\rr{d}} \partial_y^\alpha\!\Big(e^{-2\pi i\,\xi\cdot y}\,e^{\pi i\,y\cdot Q_1y}\Big)\, f(y)\,dy.
			\label{eq:D-ibp}
		\end{align}
		Since the phase is quadratic in $y$, $\partial_y^\alpha\big(e^{-2\pi i\xi\cdot y}e^{\pi i y\cdot Q_1y}\big)$ is the same exponential times  polynomials
		in $(\xi,y)$ of total degree $|\alpha|$:
		\begin{align*}\partial_y^\alpha\big(e^{-2\pi i\xi\cdot y}e^{\pi i y\cdot Q_1y}\big)&=\sum_{\beta\leq \alpha}\binom{\alpha}{\beta}\partial^\beta e^{-2\pi i\xi\cdot y}\partial^{\alpha-\beta}e^{\pi i y\cdot Q_1y}\\
			&=\sum_{\beta\leq \alpha}\binom{\alpha}{\beta}(-2\pi i \xi)^\beta e^{-2\pi i\xi\cdot y}P_{\alpha-\beta}(y)e^{\pi i y\cdot Q_1y}
		\end{align*}
		with $P_{\alpha-\beta}(y)$ polynomial of degree $|\alpha-\beta|$ with coefficient depending on $ Q_1$.
		Hence,
		\begin{align*}
			\mathscr{F}(\mathfrak p_{Q_1}\partial^\alpha f)(\xi)
			&=
			(-2\pi i)^{|\alpha|}\sum_{\beta\leq \alpha}\binom{\alpha}{\beta}(-2\pi i \xi)^\beta\,\int_{\rr{d}}P_{\alpha-\beta}(y)\,e^{-2\pi i\,\xi\cdot y}\,e^{\pi i\,y\cdot Q_1y}f(y)dy\\
			&=	(-2\pi i)^{|\alpha|}\sum_{\beta\leq \alpha}\binom{\alpha}{\beta}(-2\pi i \xi)^\beta\,\mathscr{F}(\mathfrak p_{Q_1} P_{\alpha-\beta}f)(\xi).
		\end{align*}
		%		Now,
		%		$$\mathscr{F}(\mathfrak p_{Q_1} P_{\alpha-\beta}f)(\xi)=\mathscr{F}( P_{\alpha-\beta}\mathfrak p_{Q_1}f)(\xi)=\sum_{\gamma\leq \alpha-\beta} c_{\alpha,\beta}\partial^\gamma_\xi\mathscr{F}(\mathfrak p_{Q_1}f)(\xi)$$
		This yields,
		$$
		\|\mathscr{F}(\mathfrak p_{Q_1}\partial^\alpha f)\|_{L^1_{v_s}}\leq 2^{|\alpha|}(2\pi)^{2|\alpha|} \sum_{\beta\leq \alpha}\|\mathscr{F}(\mathfrak p_{Q_1} P_{\alpha-\beta}f)\|_{L^1_{v_{s+|\beta|}}} \lesssim \sum_{\beta\leq \alpha}\|P_{\alpha-\beta}f\|_{\mathscr B^S_{s+|\beta|}}.
		$$
		From the previous estimate and the dilation bound \eqref{eq:D-step0} we infer
		$$
		\|\partial^\alpha f\|_{\mathscr B^{S}_{s}}=\|\widehat S(\partial^\alpha f)\|_{L^{1}_{v_s}}\leq 2^{|\alpha|}(2\pi)^{2|\alpha|} 
		|\det B|^{\frac{1}{2}}\,D_{s}(B^{-1})\sum_{\beta\leq \alpha}\|P_{\alpha-\beta}f\|_{\mathscr B^S_{s+|\beta|}},
		$$ 
		as desired.
	\end{proof}
	The $L^p$-boundedness properties of metaplectic operators have been completely characterized in \cite{Giacchi2024Lp}. Let
	$S$ with the block decomposition \eqref{blockS}.
	Then the following statements hold \cite[Theorem 1.2]{Giacchi2024Lp}:
	
	\begin{theorem}
		\label{thm:Giacchi-Lp}
		Let $\widehat S \in Mp(d,\mathbb{R})$ be a metaplectic operator with   symplectic projection
		$S$ having the block decomposition \eqref{blockS}. Then,
		\begin{enumerate}
			\item[(i)] If $B=0$, then $\widehat S$ extends to a bounded operator on 
			$L^r(\mathbb{R}^d)$ for every $1 \le r \le \infty$. 
			In fact, in this case $\widehat S$ is a homeomorphism on $L^r(\mathbb{R}^d)$ 
			for every $1 \le r \le \infty$.
			
			\item[(ii)] If $\det B \neq 0$ (i.e., $\widehat S$ is free), then 
			$\widehat S$ extends to a bounded operator
			\begin{equation*}
				\widehat S : L^r(\mathbb{R}^d) \longrightarrow L^q(\mathbb{R}^d)
			\end{equation*}
			if and only if
			\begin{equation*}
				1 \le r \le 2,
				\qquad
				\frac{1}{r} + \frac{1}{q} = 1.
			\end{equation*}
			Equivalently, $\widehat S : L^r \to L^{r'}$ is bounded for 
			$1 \le r \le 2$, where $r'$ denotes the conjugate exponent of $r$.
		\end{enumerate}
	\end{theorem}

	In particular, if $\widehat{S}$ is free, the operator 
	$\widehat{S}$ cannot be bounded on $L^1(\rr{d})$, nor on $L^r(\rr{d})$ for any $r\neq 2$. 
	In 
	Theorem 4.1 from Giacchi's paper \cite{Giacchi2024Lp} the precise operator norm estimates (Hausdorff-Young type bounds)
	are explicitly computed, as follows.
	\begin{theorem}
		\label{thm:Giacchi-precise}
		Let $\widehat S \in Mp(d,\mathbb{R})$ be a free metaplectic operator with symplectic projection
		$S$ having the block decomposition \eqref{blockS}.	
		Then for every $1 \le r \le 2$ and $r'$ such that $1/r+1/r'=1$,
		$\widehat S$ extends to a bounded operator
		\begin{equation*}
			\widehat S : L^r(\mathbb{R}^d) \longrightarrow L^{r'}(\mathbb{R}^d),
		\end{equation*}
		and the following estimate holds:
		\begin{equation}
			\label{eq:Giacchi-precise}
			\|\widehat S f\|_{L^{r'}}
			\le 
			C_{r,d}\,
			|\det B|^{\,\frac{1}{r}-\frac{1}{2}}
			\,
			\|f\|_{L^r},
			\qquad 1\le r \le 2,
		\end{equation}
		where \begin{equation}\label{eq:Cpd}
			C_{r,d}=
			\left(\frac{r^{\frac{1}{2r}}}{(r')^{\frac{1}{2r'}}}\right)^{\frac{d}{2}}.
		\end{equation} denotes the Hausdorff-Young constant depending only on $r$ and $d$.
	\end{theorem}
	\begin{remark}
		When $r=2$, estimate \eqref{eq:Giacchi-precise} reduces to unitarity of 
		metaplectic operators on $L^2(\mathbb{R}^d)$. 
		When $r=1$, one recovers the endpoint bound 
		$\|\widehat S f\|_{L^\infty}
		\le  |\det B|^{1/2} \|f\|_{L^1}$,
		which is the analogue of the $L^1$-$L^\infty$ bound for the Fourier transform.
	\end{remark}
	Roughly speaking, metaplectic operators behave badly on $L^r$ spaces, as for their special case given by the Fourier transform. The suitable function spaces for these operators, extending the $L^2$ case, are modulation spaces, see \cite{FUHR2024101604}. 
	
	\begin{theorem}[Local Sobolev control by generalized Barron norms]
		Let
		$n\in\mathbb N$,
		$r\ge2$, and $S\in \mathrm{Sp}(d,\rr)$ be free, with block decomposition \eqref{blockS}, and 
		$\widehat S$  the associated metaplectic operator.
		Assume that $\Omega\subset \rd$ be bounded and open with non-empty interior and
		\begin{equation}\label{eq:chicomp}
			\chi_\Omega \in \mathscr F L^r(\rd).\end{equation}
		Then for every $f\in \mathscr S(\rd)$,
		\begin{equation}
			\label{eq:local-sobolev-generalized}
			\|f\|_{W^{n,r}(\Omega)}
			\le
			C_{n,d,B}
			\|\chi_\Omega\|_{\mathscr F L^r}
			\sum_{|\gamma|\le n}
			\|P_\gamma f\|_{\mathscr B^S_n}.
		\end{equation}
		where
		$P_{\gamma}$ is the polynomial of degree $|\gamma|$ appearing in Proposition~\ref{lem:derivative-control},
		and 
		\begin{equation*}
			C_{n,r,d,B}
			=C_{r',d}\,
			|\det B|^{\frac12}\binom{d+n-1}{n}
			\frac{(8\pi^2)^{n+1}-1}{8\pi^2-1} .
		\end{equation*}
		where $C_{r',d}$ is defined in \eqref{eq:Cpd} (with $r'$ in place of $r$).
	\end{theorem}
	
	\begin{proof}
		We argue as in \cite[Proposition 12]{Abdeljawad25TimeFrequencyAnalysisNeural}.
		Since $\Omega$ is bounded with non-empty interior, the Sobolev norm is equivalent to
		\begin{equation*}
			\|f\|_{W^{n,r}(\Omega)}
			\asymp
			\sum_{|\alpha|\le n}
			\|\chi_\Omega\,\partial^\alpha f\|_{L^r(\rd)}.
		\end{equation*}
		Following \cite[Lemma 2.4 and Definition 3.1]{Tartar07IntroductionSobolevSpaces}, we introduce a standard family of mollifiers. Let
		$0<\varepsilon<1$ and define
		\[
		\rho_\varepsilon(x)
		=
		\frac{1}{\varepsilon^d\|\phi\|_{L^1}}
		\phi\!\left(\frac{x}{\varepsilon}\right),
		\]
		where
		\[
		\phi(x)
		=
		\exp\!\left(-\frac{1}{1-|x|^2}\right)\chi_{B_1(0)}(x).
		\]
		Here \(B_1(0)\) denotes the closed unit ball. By construction,
		\[
		\rho_\varepsilon\in C_c^\infty(\mathbb R^d),\qquad
		\operatorname{supp}\rho_\varepsilon=\overline{B_\varepsilon(0)},
		\qquad
		\|\rho_\varepsilon\|_{L^1}=1.
		\]
		Moreover,
		\[
		\|\rho_\varepsilon\|_{L^2}
		\le
		\varepsilon^{-d/2},
		\]
		which follows from a standard change of variables together with the estimate
		\(\rho_1(x)\le1\).
		
		Given a domain \(\Omega\subset\mathbb R^d\), we define its regularized characteristic function by
		\[
		\chi_\Omega^\varepsilon
		=
		\chi_\Omega*\rho_\varepsilon.
		\]
		Since the support of a convolution is contained in the Minkowski sum of the supports,
		\[
		\operatorname{supp}\chi_\Omega^\varepsilon
		\subseteq
		\overline{\operatorname{supp}\chi_\Omega+\operatorname{supp}\rho_\varepsilon}
		\subseteq
		\Omega_\varepsilon,
		\]
		where
		\[
		\Omega_\varepsilon
		=
		\{x\in\mathbb R^d:\operatorname{dist}(x,\Omega)\le\varepsilon\}.
		\]
		Finally, \cite[Proposition A.2]{Abdeljawad23SpaceTimeApproximationShallow} proves that, for every locally \(L^r\)-function \(h\),
		\begin{equation}\label{eq:limit}
			\lim_{\varepsilon\to0}
			\|\chi_\Omega^\varepsilon h\|_{L^r}
			=
			\|\chi_\Omega h\|_{L^r}.
		\end{equation}
		For $r\geq2$, by the  inequality \eqref{eq:Giacchi-precise},  for every $|\alpha|\leq n$,
		\begin{align*}\|\chF{\Omega}^\epsilon\partial^\alpha f\|_{L^r}&=\|\widehat{S^{-1}}\widehat{S}(
			\chF{\Omega}^\epsilon\partial^\alpha f)\|_{L^r}\leq C_{r',d}|\det B^{T}|^{\frac1{r'}-\frac12}\|\widehat{S}(
			\chF{\Omega}^\epsilon\partial^\alpha f)\|_{L^{r'}}.\end{align*}

		By Lemma \ref{lem:M-localization}, applied with $u=\partial^\alpha f$, $r'$ in place of $r$ and $s=0$,
		\begin{equation*}
			\label{eq:cutoff-step}
			\|\widehat{S}(\chF{\Omega}^\epsilon\partial^\alpha f)\|_{L^{r'}}
			\le			
			\|\chi_\Omega^\epsilon\|_{\mathscr{F} L^{r'}}
			\|\partial^\alpha f\|_{\mathscr B^{S}_0}.
		\end{equation*}
		By Proposition \ref{lem:derivative-control} with $s=0$,
		\begin{equation*}
			\|\partial^\alpha f\|_{\mathscr B^{S}_{0}}
			\le
			2^{|\alpha|}
			(2\pi)^{2|\alpha|}
			|\det B|^{\frac12}
			\sum_{\beta\le \alpha}
			\|P_{\alpha-\beta}f\|_{\mathscr B^{S}_{|\beta|}}\leq 2^{|\alpha|}
			(2\pi)^{2|\alpha|}
			|\det B|^{\frac12}	\sum_{\beta\le \alpha}\|P_{\alpha-\beta}f\|_{\mathscr B^{S}_{|\alpha|}}.
		\end{equation*}
		As  observed in  
		\cite[ formula (2.10)]{Abdeljawad23SpaceTimeApproximationShallow}, 
		$$	\|\chi_\Omega^\epsilon\|_{\mathscr{F} L^{r'}}=\|\mathscr{F}(\chi_\Omega \ast\rho_\epsilon)\|_{L^{r'}} \leq \|\mathscr{F}(\chi_\Omega )\|_{L^{r'}}\|\mathscr{F}(\rho_\epsilon)\|_{L^1}=\|\mathscr{F}(\chi_\Omega )\|_{L^{r'}}. $$
		Combining \eqref{eq:cutoff-step} with the derivative estimate and using \eqref{eq:limit} we obtain
		\begin{equation*}
			\|\chi_\Omega\,\partial^\alpha f\|_{L^r}
			\le 2^{|\alpha|}
			(2\pi)^{2|\alpha|} C_{r',d} |\det B|^{\frac1{r'}}
			\|\chi_\Omega\|_{\mathscr{F} L^{r'}}
			\sum_{\beta\le \alpha}
			\|P_{\alpha-\beta}f\|_{\mathscr B^{S}_{|\alpha|}}.
		\end{equation*}
		Summing over $|\alpha|\le n$, and using
		\begin{align*}
			\sum_{|\alpha|\leq n}2^{3|\alpha|}\pi^{2|\alpha|}
			&=
			\sum_{k=0}^{n}\sum_{|\alpha|=k}2^{3k}\pi^{2k}
			=
			\sum_{k=0}^{n}(8\pi^2)^k\binom{d+k-1}{k} \\
			&\leq
			\binom{d+n-1}{n}\sum_{k=0}^{n}(8\pi^2)^k \\
			&=
			\binom{d+n-1}{n}
			\frac{(8\pi^2)^{n+1}-1}{8\pi^2-1}.
		\end{align*} yields \eqref{eq:local-sobolev-generalized}.
	\end{proof}
	The action of $\widehat{S}$ on  $\partial^\alpha f$ can be controlled by  the classical Barron norm of  $f$ as follows.
	\begin{lemma}\label{lem:M-localization2}
		Fix $f\in\mathscr S(\rr{d})$ and $\alpha\in \zz{d}_+$ with $|\alpha|>0$. Assume $\widehat{S}$ to be a free metaplectic operator with projection $S$ having block decomposition \eqref{blockS}. Then, for every $f\in\mathscr{S}(\rd)$ one has
		$$	\|\widehat{S}(\partial^\alpha f)\|_{L^\infty_{v_s}}
		\leq |\det B|^{-\frac{1}{2}}\,D_{s}(B^{-1})\|f\|_{\mathscr B_{|s|+|\alpha|}},$$
		for every $s\leq0$.
	\end{lemma}
	\begin{proof}
		For $f\in\mathscr S(\rr{d})$, using \eqref{eq:dilq} we have
		\begin{equation}\label{eq:D-step00}
			\|\widehat S(\partial^\alpha f)\|_{L^{\infty}_{v_s}}
			=
			\|\mathfrak{T}_{B^{-1}}\mathscr{F}(\mathfrak p_{Q_1}\partial^\alpha f)\|_{L^{\infty}_{v_s}}
			\le
			|\det B|^{-\frac{1}{2}}\,D_{s}(B^{-1})\,
			\|\mathscr{F}\big(\mathfrak{p}_{Q_1}\partial^\alpha f\big)\|_{L^{\infty}_{v_s}}.
		\end{equation}
		Now, 
		\begin{align*}
			\|\mathscr{F}\big(\mathfrak{p}_{Q_1}\partial^\alpha f\big)\|_{L^{\infty}_{v_s}}
			&\le \|\mathscr{F}(\mathfrak{p}_{Q_1})\ast\mathscr{F}\partial^\alpha f\|_{L^{\infty}_{v_s}}\leq  \|(\mathfrak{p}_{-Q_1})\ast (2\pi i \xi)^\alpha \mathscr{F}f\|_{L^{\infty}_{v_s}}
			\\
			&\leq  \|(\mathfrak{p}_{-Q_1})\|_{L^{\infty}_{v_{s}}} \|(2\pi i \xi)^\alpha \mathscr{F}f\|_{L^{1}_{v_{|s|}}}\leq \|(2\pi i \xi)^\alpha \mathscr{F}f\|_{L^{1}_{v_{|s|}}}\leq
			\| \mathscr{F}f\|_{L^{1}_{v_{|s|+|\alpha|}}},
		\end{align*}
		as desired.
	\end{proof}  
	The following inversion formula will be used in the approximation theorems below.
	\begin{lemma}[Inversion formula for metaplectic Barron spaces]
		\label{lem:metaplectic-barron-inversion}
		Let $S\in \mathrm{Sp}(d,\mathbb R)$ be free, with block decomposition \eqref{blockS}
		and let $\widehat S$ be the associated metaplectic operator, with the
		phase convention fixed in \eqref{defGeneratingFunction}.
		If $s\ge 0$ and $f\in \mathcal B^S_s(\mathbb R^d)$, then
		$\widehat S^{-1}(\widehat S f)\in L^\infty(\mathbb R^d)$ and,  for a.e. $x\in\mathbb R^d$,
		\begin{equation}\label{eq:invformet}
			f(x)
			=
			\frac{1}{\sqrt{|\det B|}}
			\int_{\mathbb R^d}
			e^{-\pi i\left[x\cdot (B^{-1}A)^T x
				+
				\xi\cdot (DB^{-1})^T \xi \right]}
			\,\widehat{S} f(\xi)\,
			e^{2\pi i\, x\cdot B^{-T}\xi}\,
			d\xi.
		\end{equation}
	\end{lemma}
	\begin{proof} Observe that $f\in \mathscr{S}'$
		and  the inversion formula \cite{degosson2011symplectic} holds
		\[
		f=\widehat S^{-1}(\widehat S f)
		\qquad \text{in } \mathscr S'(\mathbb R^d).
		\] 
		Since $s\ge 0$, one has $v_s\ge 1$. Therefore
		\[
		\widehat S f\in L^1_{v_s}(\mathbb R^d)
		\subseteq L^1(\mathbb R^d).
		\]
		By the $L^1$-$L^\infty$ boundedness of free metaplectic operators in Theorem \ref{thm:Giacchi-Lp},
		applied to $\widehat S^{-1}$, the function $\widehat S^{-1}(\widehat S f)$
		belongs to $L^\infty(\mathbb R^d)$. Hence the displayed formula \eqref{eq:invformet} gives an
		$L^\infty$-representative of $f$.
	\end{proof}
	
%%%%%%%%%%%%%%%%%%%%%%%%%%%%%%%%%%%%%%%%%%%%%%%%%%%%%%%%%%%%%%%%%%
\section{Convergence Rates for Approximation of General Barron Spaces}
\label{sec:space_approximation}
%%%%%%%%%%%%%%%%%%%%%%%%%%%%%%%%%%%%%%%%%%%%%%%%%%%%%%%%%%%%%%%%%%

In this section, we prove approximation estimates for finite-width models
associated with metaplectic phase-space representations. More precisely,
given a free symplectic matrix $S\in\mathrm{Sp}(d,\rr)$ as in
\eqref{blockS}, we consider the neural metaplectic dictionary
consisting of the chirped ridge-type atoms
\begin{align*}
\bb{D}_S
=
\left\{
x\mapsto
\sigma\left(
\omega\cdot B^{-1}x+b
\right)
\cos\left(
\pi x\cdot(B^{-1}A)x-\theta
\right)
:
(\omega,b,\theta)
\in
\rd\times\rr\times[0,2\pi)
\right\},
\end{align*}
where $\sigma$ is a real-valued non-linear function.

We derive Monte-Carlo approximation bounds for real-valued functions in
the metaplectic Barron space introduced in \cref{def:metabarron}, using
finite linear combinations of atoms from $\bb{D}_S$. We first consider
approximation on bounded domains and then extend the analysis to unbounded
domains. In particular, \cref{thm:MC_metaplectic_Wmr_noeps} establishes
the dimension-independent exponent $N^{-1/2}$ in $W^{n,r}(\Omega)$,
whereas \cref{thm:metaplectic_unbounded_weighted} provides the
corresponding extension to unbounded domains.

\begin{theorem}[Monte-Carlo approximation]
\label{thm:MC_metaplectic_Wmr_noeps}
Let $\Omega\subset\rd$ be a bounded open set, let
$n\in\zzp{}$, $2\le r<\infty$, and $s>1$. Consider a real-valued
activation function
$\sigma\in W^{k,\infty}(v_s)\setminus{\{0\}}$, where $k\ge n$.
Let $S\in\mathrm{Sp}(d,\rr)$ be a free symplectic matrix written as
in \eqref{blockS}. Then there exists a constant $C>0$ such that, for
every real-valued $f\in\mathscr B^S_{n+1}(\rd)$ and every
$N\in\nn{}$,
\begin{equation}\label{eq:MC_rate_Wmr_noeps}
\inf_{f_N\in\Sigma_N(\bb{D}_S)}
\|f-f_N\|_{W^{n,r}(\Omega)}
\le
CN^{-1/2}
\|f\|_{\mathscr B^S_{n+1}(\rd)}.
\end{equation}
\end{theorem}
	
\begin{proof}
	We divide the proof into five steps.
	
	\emph{Step 1: Integral representation of the activation and the target functions.}
	Set
	$$
	M:=B^{-1}A,
	\qquad
	R:=DB^{-1},
	\qquad
	L:=BB^{-T}.
	$$
	By \eqref{eq:ABtrasp} and \eqref{eq:DBtrasp}, the matrices $M$ and $R$
	are symmetric. Moreover,
	$$
	\det L
	=
	\det(B)\det(B^{-T})
	=1.
	$$
	The matrix $L$ is introduced solely to reconcile the corrected inverse
	plane wave with the unchanged ridge variable in $\bb D_S$. Indeed, for
	all $x,\xi\in\rd$ and $z\in\rr$,
	\begin{equation}
		x\cdot B^{-1}(zL\xi)
		=z x^TB^{-T}\xi
		=z\xi\cdot B^{-1}x.
		\label{eq:L-ridge-identity}
	\end{equation}
	Since $f\in\mathscr B^S_{n+1}(\rd)$, one has
	$\widehat Sf\in L^1(\rd)$. Hence Lemma
	\ref{lem:metaplectic-barron-inversion} yields, for almost every $x\in\rd$,
	\begin{equation}
		f(x)
		=
		\frac1{\sqrt{|\det B|}}
		\int_{\rd}
		e^{-\pi i x\cdot Mx}
		e^{2\pi i x\cdot B^{-1}\eta}
		e^{-\pi i\eta\cdot R\eta}
		\widehat Sf(\eta)\,d\eta.
		\label{eq:expf}
	\end{equation}
	The integral in \eqref{eq:expf} is absolutely convergent, because all
	three exponential factors have modulus one.
	
	We next represent the plane wave by translates of $\sigma$. Since
	$\sigma\in W^{k,\infty}(v_s;\rr)$ and $s>1$,
	$$
	|\sigma(t)|
	\le
	\|\sigma\|_{L^\infty_{v_s}}v_{-s}(t),
	$$
	so $\sigma\in L^1(\rr)$. Consequently, $\widehat\sigma$ is continuous.
	Since $\sigma\not\equiv0$, uniqueness of the Fourier transform implies
	$\widehat\sigma\not\equiv0$. Choose $z_0\in\rr$ such that
	$\widehat\sigma(z_0)\ne0$. If $z_0\ne0$, set $z=z_0$, if $z_0=0$,
	continuity gives a nonzero $z$ sufficiently close to the origin with
	$\widehat\sigma(z)\ne0$.
	
	For every $a\in\rr$, the change of variables $t=a+b$ gives
	\begin{equation}
		e^{2\pi iaz}
		=
		\frac1{\widehat\sigma(z)}
		\int_{\rr}\sigma(a+b)e^{-2\pi ibz}\,db.
		\label{eq:activation-plane-wave}
	\end{equation}
	Taking $a=\xi\cdot B^{-1}x$ and using
	\eqref{eq:L-ridge-identity}, we obtain
	\begin{equation}
		e^{2\pi i x\cdot B^{-1}(zL\xi)}
		=
		\frac1{\widehat\sigma(z)}
		\int_{\rr}
		\sigma(\xi\cdot B^{-1}x+b)e^{-2\pi ibz}\,db.
		\label{eq:dividesigma}
	\end{equation}
	Substituting \eqref{eq:dividesigma} into \eqref{eq:expf} and changing
	variables $\eta=zL\xi$, whose Jacobian has absolute value $|z|^d$ because
	$|\det L|=1$, yields
	\begin{equation}
		\begin{aligned}
			f(x)
			&=
			\frac{|z|^d}{\widehat\sigma(z)\sqrt{|\det B|}}
			\int_{\rd}\int_{\rr}
			\rho_S(x,\xi,b)e^{-2\pi ibz}
			e^{-\pi i(zL\xi)\cdot R(zL\xi)}
			\widehat Sf(zL\xi)\,db\,d\xi,
		\end{aligned}
		\label{eq:bounded-complex-representation}
	\end{equation}
	where
	\begin{equation}
		\rho_S(x,\xi,b)
		:=
		\sigma(\xi\cdot B^{-1}x+b)e^{-\pi i x\cdot Mx}.
		\label{eq:metaplectic-activation}
	\end{equation}
	
	\emph{Step 2: A finite coefficient measure.}
	Define
	\begin{equation}
		\omega(\xi,b)
		:=
		\frac{v_n(\xi)}
		{v_s((|b|-R_{\Omega,B}|\xi|)_+)},
		\qquad
		R_{\Omega,B}:=\sup_{x\in\Omega}|B^{-1}x|,
		\label{eq:omega_def}
	\end{equation}
	and set
	$$
	\widetilde\rho_S(x,\xi,b)
	:=
	\frac{\rho_S(x,\xi,b)}{\omega(\xi,b)}.
	$$
	Since $\Omega$ is bounded and $B$ is invertible, $R_{\Omega,B}<\infty$,
	moreover, $\omega(\xi,b)>0$ for every $(\xi,b)$. Define the complex
	measure $\lambda_f$ on $\rd\times\rr$ by
	\begin{equation}
		\begin{aligned}
			d\lambda_f(\xi,b)
			:=
			&\frac{|z|^d}{\widehat\sigma(z)\sqrt{|\det B|}}
			e^{-2\pi ibz}\omega(\xi,b)
			e^{-\pi i(zL\xi)\cdot R(zL\xi)}
			\widehat Sf(zL\xi)\,db\,d\xi.
		\end{aligned}
		\label{eq:bounded-coefficient-measure}
	\end{equation}
	Then \eqref{eq:bounded-complex-representation} becomes
	\begin{equation}
		f
		=
		\int_{\rd\times\rr}
		\widetilde\rho_S(\cdot,\xi,b)\,d\lambda_f(\xi,b).
		\label{eq:Bochner-integral-in-Sobolev-space}
	\end{equation}
	We first show that $\lambda_f$ has finite total variation. For fixed
	$\xi$, put $a_\xi:=R_{\Omega,B}|\xi|$. Since $s>1$,
	$$
	\begin{aligned}
		\int_{\rr}\omega(\xi,b)\,db
		&=
		2v_n(\xi)
		\left(
		a_\xi+\int_0^\infty(1+t)^{-s}\,dt
		\right)\\
		&=
		2v_n(\xi)
		\left(
		R_{\Omega,B}|\xi|+\frac1{s-1}
		\right)\\
		&\le
		C_{\Omega,B,s}v_{n+1}(\xi).
	\end{aligned}
	$$
	Therefore,
	$$
	\|\lambda_f\|
	\le
	C
	\int_{\rd}
	v_{n+1}(\xi)|\widehat Sf(zL\xi)|\,d\xi.
	$$
	With $\eta=zL\xi$, one has $d\xi=|z|^{-d}d\eta$, and the dilation
	constant in \eqref{Dsq} gives
	$$
	v_{n+1}((zL)^{-1}\eta)
	\le
	D_{n+1}(zL)^{-T})v_{n+1}(\eta).
	$$
	It follows that
	\begin{equation}
		\|\lambda_f\|
		\le
		C
		\int_{\rd}
		v_{n+1}(\eta)|\widehat Sf(\eta)|\,d\eta
		=
		C\|f\|_{\mathscr B^S_{n+1}(\rd)}.
		\label{eq:barromeasures}
	\end{equation}
	Here and below, the constant may depend on the fixed choice of $z$ through
	$|\widehat\sigma(z)|^{-1}$ and $D_{n+1}((zL)^{-T})$.

	\emph{Step 3: Embedding of the metaplectic Barron space into the variation space.}
	Note that, by
	Lemma \ref{lem:dictionary-uniform-bound},
	\begin{equation}
		K_{\widetilde\rho_S}
		:=
		\sup_{(\xi,b)\in\rd\times\rr}
		\|\widetilde\rho_S(\cdot,\xi,b)\|_{W^{n,r}(\Omega)}
		<\infty.
		\label{eq:bounded-complex-atom-bound}
	\end{equation}
	Here, we verify the required measurability rather than leaving it implicit.
	Choose Borel representatives of $\sigma$ and of its weak derivatives up to
	order $n$. The derivative formula in the proof of
	Lemma \ref{lem:dictionary-uniform-bound} shows that, for every
	$|\alpha|\le n$, the function
	$$
	(x,\xi,b)
	\longmapsto
	\partial_x^\alpha\widetilde\rho_S(x,\xi,b)
	$$
	is measurable. Since $r<\infty$, each corresponding parameter map
	into $L^r(\Omega)$ is strongly measurable. Applying this observation to the
	finite derivative vector
	$$
	\mathcal J\widetilde\rho_S
	:=
	(\partial_x^\alpha\widetilde\rho_S)_{|\alpha|\le n}
	$$
	shows that
	$(\xi,b)\mapsto\widetilde\rho_S(\cdot,\xi,b)$ is strongly measurable as a
	$W^{n,r}(\Omega)$-valued map. In view of
	\eqref{eq:bounded-complex-atom-bound} and \eqref{eq:barromeasures},
	$$
	\int_{\rd\times\rr}
	\|\widetilde\rho_S(\cdot,\xi,b)\|_{W^{n,r}(\Omega)}
	\,d|\lambda_f|(\xi,b)
	\le
	K_{\widetilde\rho_S}\|\lambda_f\|
	<\infty.
	$$
	Thus the integral in
	\eqref{eq:Bochner-integral-in-Sobolev-space} is a well-defined Bochner
	integral in the Sobolev space.
	
	\emph{Step 4: Passage to the real-valued variation space.}
	By the polar decomposition theorem for complex measures
	\cite[Theorem 6.12]{RudinRealComplexAnalysis}, there exists a measurable
	function $h_f$ with $|h_f|=1$ $|\lambda_f|$-almost everywhere such that
	$$
	d\lambda_f=h_f\,d|\lambda_f|.
	$$
	Choosing the principal argument we may write
	$$
	h_f(\xi,b)=e^{i\theta_f(\xi,b)},
	\qquad
	\theta_f:\rd\times\rr\to[0,2\pi),
	$$
	with $\theta_f$ measurable.
	
	Since $f$ is real-valued taking real parts in
	\eqref{eq:Bochner-integral-in-Sobolev-space} gives
	$$
	f(x)
	=
	\int_{\rd\times\rr}
	\frac{\sigma(\xi\cdot B^{-1}x+b)}{\omega(\xi,b)}
	\cos\bigl(\pi x\cdot Mx-\theta_f(\xi,b)\bigr)
	\,d|\lambda_f|(\xi,b).
	$$
	This identity is only an intermediate representation. To formulate it
	using a fixed dictionary independent of the target function $f$, set
	$$
	\Theta
	:=
	\rd\times\rr\times[0,2\pi)
	$$
	and define, for $(\xi,b,\theta)\in\Theta$, the real-valued weighted atom
	$$
	\widetilde\varrho_S(x,\xi,b,\theta)
	:=
	\frac{\sigma(\xi\cdot B^{-1}x+b)}{\omega(\xi,b)}
	\cos\bigl(\pi x\cdot Mx-\theta\bigr).
	$$
	The associated weighted dictionary is
	$$
	\widetilde{\bb D}_S
	:=
	\left\{
	\widetilde\varrho_S(\cdot,\xi,b,\theta):
	(\xi,b,\theta)\in\Theta
	\right\}.
	$$
	In particular, neither the atoms nor the dictionary depend on $f$.
	
	Since
	$$
	\widetilde\varrho_S(\cdot,\xi,b,\theta)
	=
	\operatorname{Re}\left(
	e^{i\theta}\widetilde\rho_S(\cdot,\xi,b)
	\right),
	$$
    we have
	\begin{equation}
		\sup_{(\xi,b,\theta)\in\Theta}
		\|\widetilde\varrho_S(\cdot,\xi,b,\theta)\|_{W^{n,r}(\Omega)}
		\le
		K_{\widetilde\rho_S}
		<\infty.
		\label{eq:bounded-real-atom-bound}
	\end{equation}
	Moreover, by the same derivative-vector argument used in Step 3, the map
	$$
	\Theta\ni(\xi,b,\theta)
	\longmapsto
	\widetilde\varrho_S(\cdot,\xi,b,\theta)
	\in W^{n,r}(\Omega)
	$$
	is strongly measurable.
	
	We now encode the target-dependent phase in the measure rather than in
	the atom. Define a positive measure $\mu_f$ on $\Theta$ by
	$$
	\mu_f(E)
	:=
	\int_{\rd\times\rr}
	\mathbf 1_E\bigl(\xi,b,\theta_f(\xi,b)\bigr)
	\,d|\lambda_f|(\xi,b)
	$$
	for every Borel set $E\subseteq\Theta$. Equivalently,
	$$
	d\mu_f(\xi,b,\theta)
	=
	d|\lambda_f|(\xi,b)\,
	\delta_{\theta_f(\xi,b)}(d\theta).
	$$
	The measurability of $\theta_f$ guarantees that $\mu_f$ is well defined.
	Furthermore,
	$$
	\mu_f(\Theta)
	=
	|\lambda_f|(\rd\times\rr)
	=
	\|\lambda_f\|.
	$$
	The uniform estimate
	\eqref{eq:bounded-real-atom-bound} therefore implies
	$$
	\int_\Theta
	\|\widetilde\varrho_S(\cdot,\xi,b,\theta)\|_{W^{n,r}(\Omega)}
	\,d\mu_f(\xi,b,\theta)
	\le
	K_{\widetilde\rho_S}\|\lambda_f\|
	<\infty.
	$$
	Hence the following Bochner integral is well defined in
	$W^{n,r}(\Omega)$, and the preceding real-part identity can be written as
	$$
	f
	=
	\int_\Theta
	\widetilde\varrho_S(\cdot,\xi,b,\theta)
	\,d\mu_f(\xi,b,\theta).
	$$
	
	Applying Proposition \ref{prop:variation_norm} directly to this
	parameter-space representation yields
	\begin{equation}
		\|f\|_{\mathcal K(\widetilde{\bb D}_S)}
		\le
		\mu_f(\Theta)
		=
		\|\lambda_f\|
		\le
		C\|f\|_{\mathscr B^S_{n+1}(\rd)}.
		\label{eq:bounded-variation-bound}
	\end{equation}

	\emph{Step 5: Maurey's estimate.}
	Since $r\ge2$, the space $L^r(\Omega)$ is of type $2$. 
    Thus $W^{n,r}(\Omega)$ is of type $2$, see
	also \cite[Corollary A.6]{Brzezniak95StochasticPartialDifferential}.
	With
	$$
	M_f:=\|f\|_{\mathcal K(\widetilde{\bb D}_S)},
	$$
	Proposition \ref{prop:approximation_type2}, together with
	\eqref{eq:bounded-real-atom-bound} and
	\eqref{eq:bounded-variation-bound}, yields
	$$
	\inf_{g_N\in\Sigma_{N,M_f}(\widetilde{\bb D}_S)}
	\|f-g_N\|_{W^{n,r}(\Omega)}
	\le
	CN^{-1/2}\|f\|_{\mathscr B^S_{n+1}(\rd)}.
	$$
	Finally, write the atoms of the original dictionary as
	$$
	\varrho_S(x,\xi,b,\theta)
	:=
	\sigma(\xi\cdot B^{-1}x+b)
	\cos\bigl(\pi x\cdot Mx-\theta\bigr).
	$$
	Then
	$$
	\widetilde\varrho_S(x,\xi,b,\theta)
	=
	\omega(\xi,b)^{-1}\varrho_S(x,\xi,b,\theta),
	$$
	and $\omega(\xi,b)>0$. Therefore every $N$-term combination of weighted
	atoms is an $N$-term combination of atoms from $\bb D_S$, after absorbing
	$\omega(\xi,b)^{-1}$ into the corresponding scalar coefficient. We
	conclude that
	$$
	\inf_{f_N\in\Sigma_N(\bb D_S)}
	\|f-f_N\|_{W^{n,r}(\Omega)}
	\le
	CN^{-1/2}\|f\|_{\mathscr B^S_{n+1}(\rd)}.
	$$
	All constants used above may depend only on $d,n,r,s,\Omega,S$, and the fixed
	activation $\sigma$, and are independent of $f$ and $N$.
\end{proof}

\begin{theorem}[Metaplectic Barron approximation over unbounded domains]
	\label{thm:metaplectic_unbounded_weighted}
	Let $m\in\mathbb Z_+$, $2\le p<\infty$, and $s>1$. Let
	$u,\nu\in\rr$ satisfy
	$$
	s\le\nu,
	\qquad
	u>m+s+\frac{d}{p}.
	$$
	Let $\sigma\in W^{k,\infty}(v_\nu;\rr)\setminus\{0\}$ be real-valued,
	with $k\ge m$. Let $S\in\mathrm{Sp}(d,\rr)$ be free and consider the
	dictionary
	\begin{align*}
		\bb D_S
		=
		\left\{
		x\mapsto
		\sigma(\omega\cdot B^{-1}x+b)
		\cos\bigl(\pi x\cdot(B^{-1}A)x-\theta\bigr):
		(\omega,b,\theta)\in\rd\times\rr\times[0,2\pi)
		\right\}.
	\end{align*}
	Then there exists
	$C=C(d,m,p,u,\nu,s,S,\sigma)>0$ such that, for every real-valued
	$f\in\mathscr B^S_{m+s}(\rd)$ and every $N\in\nn{}$,
	\begin{equation}\label{eq:unbounded-main-rate}
		\inf_{f_N\in\Sigma_N(\bb D_S)}
		\|f-f_N\|_{W^{m,p}(v_{-u};\rd)}
		\le
		CN^{-1/2}\|f\|_{\mathscr B^S_{m+s}(\rd)}.
	\end{equation}
\end{theorem}

\begin{proof}
	We again separate the proof into the representation, the uniform atom
	estimate, the variation bound, and the final sampling argument.
	
	\emph{Step 1: Corrected integral representation.}
	Set
	$$
	M:=B^{-1}A,
	\qquad
	R:=DB^{-1},
	\qquad
	L:=BB^{-T}.
	$$
	As in the bounded-domain proof, $M$ and $R$ are symmetric,
	$\det L=1$, and
	\begin{equation}
		x\cdot B^{-1}(zL\xi)
		=
		z\xi\cdot B^{-1}x.
		\label{eq:unbounded-L-ridge-identity}
	\end{equation}
	Since $f\in\mathscr B^S_{m+s}(\rd)$ and $v_{m+s}\ge1$, one has
	$\widehat Sf\in L^1(\rd)$. Lemma
	\ref{lem:metaplectic-barron-inversion} therefore gives
	$$
	f(x)
	=
	\frac1{\sqrt{|\det B|}}
	\int_{\rd}
	e^{-\pi i x\cdot Mx}
	e^{2\pi i x\cdot B^{-1}\eta}
	e^{-\pi i\eta\cdot R\eta}
	\widehat Sf(\eta)\,d\eta.
	$$
	This integral is absolutely convergent for every $x\in\rd$.
	
	Since $s\le\nu$ and $s>1$, the activation assumption implies
	$\sigma\in L^1(\rr)$. As in the bounded-domain proof, we fix
	$z\ne0$ with $\widehat\sigma(z)\ne0$ and use
	$$
	\int_{\rr}\sigma(a+b)e^{-2\pi ibz}\,db
	=
	e^{2\pi iaz}\widehat\sigma(z).
	$$
	Taking $a=\xi\cdot B^{-1}x$, using
	\eqref{eq:unbounded-L-ridge-identity}, and changing variables
	$\eta=zL\xi$, we obtain
	\begin{align}
		f(x)
		=
		&\frac{|z|^d}{\widehat\sigma(z)\sqrt{|\det B|}}
		\int_{\rd}\int_{\rr}
		\rho_S(x,\xi,b)e^{-2\pi ibz}
		e^{-\pi i(zL\xi)\cdot R(zL\xi)}
		\widehat Sf(zL\xi)\,db\,d\xi,
		\label{eq:unbounded-complex-representation}
	\end{align}
	where
	$$
	\rho_S(x,\xi,b)
	:=
	\sigma(\xi\cdot B^{-1}x+b)e^{-\pi i x\cdot Mx}.
	$$
	For every fixed $x$, Fubini's theorem is justified by
	$$
	\int_{\rd}\int_{\rr}
	|\sigma(\xi\cdot B^{-1}x+b)|
	|\widehat Sf(zL\xi)|\,db\,d\xi
	=
	|z|^{-d}\|\sigma\|_{L^1}\|\widehat Sf\|_{L^1}
	<\infty.
	$$
	% Thus, also on the unbounded domain, $z$ and $L$ appear only in the
	% coefficient and not in the atom family.
	
	\emph{Step 2: Uniform boundedness of the weighted atoms.}
	Define
	$$
	\widetilde\rho_S(x,\xi,b)
	:=
	\frac{v_s(b)}{v_{m+s}(\xi)}\rho_S(x,\xi,b).
	$$
	We prove that
	\begin{equation}
		K_{\widetilde\rho_S}
		:=
		\sup_{(\xi,b)\in\rd\times\rr}
		\|\widetilde\rho_S(\cdot,\xi,b)\|_{W^{m,p}(v_{-u};\rd)}
		<\infty.
		\label{eq:unbounded-complex-atom-bound}
	\end{equation}
	Let $\beta$ be a multi-index with $|\beta|\le m$. Repeated
	differentiation of the chirp shows that, for every multi-index $\gamma$,
	there exists a polynomial $P_{\gamma,M}$ of degree at most $|\gamma|$,
	with coefficients depending only on $\gamma$ and $M$, such that
	$$
	\partial^\gamma e^{-\pi i x\cdot Mx}
	=
	P_{\gamma,M}(x)e^{-\pi i x\cdot Mx},
	\qquad
	|P_{\gamma,M}(x)|
	\le
	C_{\gamma,S}v_{|\gamma|}(x).
	$$
	This follows directly by induction from
	$\partial_j e^{-\pi i x\cdot Mx}=-2\pi i(Mx)_j e^{-\pi i x\cdot Mx}$.
	Moreover,
	$$
	\partial^\alpha
	\sigma(\xi\cdot B^{-1}x+b)
	=
	(B^{-T}\xi)^\alpha
	\sigma^{(|\alpha|)}(\xi\cdot B^{-1}x+b).
	$$
	Leibniz's rule therefore gives
	$$
	\begin{aligned}
		\partial^\beta\widetilde\rho_S(x,\xi,b)
		=
		&\frac{v_s(b)}{v_{m+s}(\xi)}
		\sum_{\gamma\le\beta}
		\binom{\beta}{\gamma}
		P_{\gamma,M}(x)e^{-\pi i x\cdot Mx}\\
		&\qquad\qquad\times
		(B^{-T}\xi)^{\beta-\gamma}
		\sigma^{(|\beta-\gamma|)}(\xi\cdot B^{-1}x+b).
	\end{aligned}
	$$
	Since $|e^{-\pi i x\cdot Mx}|=1$ and
	$v_{|\gamma|}(x)v_{-u}(x)\le v_{-(u-m)}(x)$ for $|\gamma|\le m$, it
	follows that
	$$
	\begin{aligned}
		\|\partial^\beta\widetilde\rho_S(\cdot,\xi,b)\|_{L^p(v_{-u};\rd)}
		\le
		&C
		\frac{v_s(b)}{v_{m+s}(\xi)}
		\sum_{\gamma\le\beta}
		|B^{-T}\xi|^{|\beta-\gamma|}
		\times
		\left\|
		\sigma^{(|\beta-\gamma|)}
		(\xi\cdot B^{-1}\cdot+b)
		\right\|_{L^p(v_{-(u-m)};\rd)}.
	\end{aligned}
	$$
	Grouping terms with $j=|\beta-\gamma|$ and using
	$|B^{-T}\xi|\le C_B|\xi|$, we obtain
	\begin{equation}
		\begin{aligned}
			\|\partial^\beta\widetilde\rho_S(\cdot,\xi,b)\|_{L^p(v_{-u};\rd)}
			\le
			C\frac{v_s(b)}{v_{m+s}(\xi)}
			\sum_{j=0}^{|\beta|}
			|\xi|^j
			\left\|
			\sigma^{(j)}(\xi\cdot B^{-1}\cdot+b)
			\right\|_{L^p(v_{-(u-m)};\rd)}.
		\end{aligned}
		\label{eq:atom-norm-estimate}
	\end{equation}
	
	It remains to estimate the ridge term. Put
	$q:=u-m$.
	The assumption $u>m+s+d/p$ gives
	$q>s+\frac{d}{p}$.
	Since $s\le\nu$ and $k\ge m$, for every $0\le j\le m$,
	$$
	|\sigma^{(j)}(t)|
	\le
	\|\sigma\|_{W^{k,\infty}(v_\nu)}v_{-\nu}(t)
	\le
	C_\sigma v_{-s}(t).
	$$
	We claim that
	\begin{equation}
		\left\|
		\sigma^{(j)}(\xi\cdot B^{-1}\cdot+b)
		\right\|_{L^p(v_{-q};\rd)}
		\le
		C
		v_{-s}\left(
		\min\left\{1,\frac1{|\xi|}\right\}|b|
		\right),
		\label{eq:unbounded-ridge-estimate-inline}
	\end{equation}
	uniformly in $(\xi,b)\in\rd\times\rr$ and $0\le j\le m$, with the
	convention that the minimum equals $1$ when $\xi=0$.
	
	If $\xi=0$, then $qp>d$ and
	$$
	\begin{aligned}
		\|\sigma^{(j)}(b)\|_{L^p(v_{-q};\rd)}
		&=
		|\sigma^{(j)}(b)|\|v_{-q}\|_{L^p(\rd)}
		\le
		Cv_{-s}(b),
	\end{aligned}
	$$
	which is \eqref{eq:unbounded-ridge-estimate-inline} in this case.
	Assume now that $\xi\ne0$ and put $a:=B^{-T}\xi$. Since $B$ is
	invertible,
	$$
	|a|\asymp_B|\xi|.
	$$
	Using the decay of $\sigma^{(j)}$, we obtain
	$$
	\left\|
	\sigma^{(j)}(a\cdot\,+b)
	\right\|_{L^p(v_{-q};\rd)}^p
	\le
	C
	\int_{\rd}
	v_{-sp}(a\cdot x+b)v_{-qp}(x)\,dx.
	$$
	Choose an orthogonal matrix $Q$ such that $Q^Ta=|a|e_1$ and set
	$x=Qy$. Then $dx=dy$, $|x|=|y|$, and $a\cdot x=|a|y_1$. Writing
	$y=(t,y')\in\rr\times\rr^{d-1}$, with the obvious interpretation when
	$d=1$, gives
	$$
	\begin{aligned}
		&\int_{\rd}
		v_{-sp}(a\cdot x+b)v_{-qp}(x)\,dx
		=
		\int_{\rr}
		v_{-sp}(|a|t+b)
		\left(
		\int_{\rr^{d-1}}
		v_{-qp}(t,y')\,dy'
		\right)dt.
	\end{aligned}
	$$
	For fixed $t$, put $T:=1+|t|$. The change of variables $y'=Tz$ and the
	inequality
	$$
	1+\sqrt{|t|^2+T^2|z|^2}
	\ge
	cT(1+|z|)
	$$
	show that
	$$
	\begin{aligned}
		\int_{\rr^{d-1}}
		v_{-qp}(t,y')\,dy'
		&\le
		CT^{d-1-qp}
		\int_{\rr^{d-1}}(1+|z|)^{-qp}\,dz\\
		&\le
		C(1+|t|)^{-qp+d-1}.
	\end{aligned}
	$$
	The last integral is finite because $qp>d-1$. Define
	$q_1:=q-\frac{d-1}{p}$.
	Then $q_1>s+1/p$, and the preceding estimate becomes
	$$
	\int_{\rr^{d-1}}
	v_{-qp}(t,y')\,dy'
	\le
	Cv_{-q_1p}(t).
	$$
	Applying \cref{lem:weight-estimate} with $\lambda=|a|$ yields
	$$
	\int_{\rd}
	v_{-sp}(a\cdot x+b)v_{-qp}(x)\,dx
	\le
	C
	v_{-sp}\left(
	\min\left\{1,\frac1{|a|}\right\}|b|
	\right).
	$$
	After taking the $p$-th root, the equivalence $|a|\asymp_B|\xi|$ and
	the elementary equivalence of the corresponding polynomial weights give
	\eqref{eq:unbounded-ridge-estimate-inline}.
	
	Combining \eqref{eq:atom-norm-estimate} with
	\eqref{eq:unbounded-ridge-estimate-inline}, and using
	$|\xi|^j\le v_m(\xi)$ for $j\le m$, gives
	$$
	\begin{aligned}
		\|\partial^\beta\widetilde\rho_S(\cdot,\xi,b)\|_{L^p(v_{-u};\rd)}
		&\le
		C
		\frac{v_s(b)v_m(\xi)}{v_{m+s}(\xi)}
		v_{-s}\left(
		\min\left\{1,\frac1{|\xi|}\right\}|b|
		\right)\\
		&=
		C
		\frac{v_s(b)}{v_s(\xi)}
		v_{-s}\left(
		\min\left\{1,\frac1{|\xi|}\right\}|b|
		\right).
	\end{aligned}
	$$
	The last expression is uniformly bounded. If $|\xi|<1$, it equals
	$v_s(\xi)^{-1}\le1$. If $|\xi|\ge1$, then
	$$
	1+|b|
	\le
	(1+|\xi|)\left(1+\frac{|b|}{|\xi|}\right),
	$$
	so
	$$
	v_s(b)
	\le
	v_s(\xi)v_s\left(\frac{|b|}{|\xi|}\right).
	$$
	Thus the last expression is again at most $1$, up to the fixed constants
	already absorbed above. Summing over $|\beta|\le m$ proves
	\eqref{eq:unbounded-complex-atom-bound}.
	
	\emph{Step 3: The coefficient measure and the real variation-space representation.}
	Define a complex measure on $\rd\times\rr$ by
	\begin{align}
		d\lambda_f(\xi,b)
		:=
		&\frac{|z|^d}{\widehat\sigma(z)\sqrt{|\det B|}}
		\frac{v_{m+s}(\xi)}{v_s(b)}
		e^{-2\pi ibz}
		e^{-\pi i(zL\xi)\cdot R(zL\xi)}
		\widehat Sf(zL\xi)\,db\,d\xi.
		\label{eq:unbounded-coefficient-measure}
	\end{align}
	Since $s>1$,
	$
	\int_{\rr}v_{-s}(b)\,db
	<\infty$.
	Hence
	$$
	\|\lambda_f\|
	\le
	C
	\int_{\rd}
	v_{m+s}(\xi)|\widehat Sf(zL\xi)|\,d\xi.
	$$
	Under $\eta=zL\xi$, one has $d\xi=|z|^{-d}d\eta$, and
	\eqref{Dsq} gives
	$$
	v_{m+s}((zL)^{-1}\eta)
	\le
	D_{m+s}((zL)^{-T})v_{m+s}(\eta).
	$$
	Consequently,
	\begin{equation}
		\|\lambda_f\|
		\le
		C
		\int_{\rd}
		v_{m+s}(\eta)|\widehat Sf(\eta)|\,d\eta
		=
		C\|f\|_{\mathscr B^S_{m+s}(\rd)}.
		\label{eq:unbounded-measure-bound}
	\end{equation}
	% The constant depends on the fixed $z$, and therefore on the fixed
	% activation $\sigma$, but not on $f$ or $N$.
	
	We now justify the vector-valued integral. Choose Borel
	representatives of the weak derivatives of $\sigma$ up to order $m$. The
	derivative formulas in Step 2 show that every map
	$$
	(x,\xi,b)
	\longmapsto
	v_{-u}(x)\partial_x^\beta
	\widetilde\rho_S(x,\xi,b),
	\qquad |\beta|\le m,
	$$
	is measurable. Since $p<\infty$, the induced parameter maps into
	$L^p(\rd)$ are strongly measurable. Applying this to the finite weighted
	derivative vector shows that
	$(\xi,b)\mapsto\widetilde\rho_S(\cdot,\xi,b)$ is strongly measurable in
	$W^{m,p}(v_{-u};\rd)$. By
	\eqref{eq:unbounded-complex-atom-bound} and
	\eqref{eq:unbounded-measure-bound},
	$$
	\int_{\rd\times\rr}
	\|\widetilde\rho_S(\cdot,\xi,b)\|_{W^{m,p}(v_{-u};\rd)}
	\,d|\lambda_f|(\xi,b)
	<\infty.
	$$
    The preceding estimate implies that the following Bochner integral is well defined:
    $$
    \int_{\rd\times\rr}
    \widetilde\rho_S(\cdot,\xi,b)
    \,d\lambda_f(\xi,b)
    \in
    W^{m,p}(v_{-u};\rd).
    $$ 
    Hence
	\begin{equation}
		f
		=
		\int_{\rd\times\rr}
		\widetilde\rho_S(\cdot,\xi,b)\,d\lambda_f(\xi,b)
		\quad\text{in }W^{m,p}(v_{-u};\rd).
		\label{eq:unbounded-Bochner-integral}
	\end{equation}
	
	By the polar decomposition theorem for complex measures
	\cite[Theorem 6.12]{RudinRealComplexAnalysis}, there exists a measurable
	function $\theta_f:\rd\times\rr\to[0,2\pi)$ such that, 
    %after alteration on a $|\lambda_f|$-null set if necessary,
	$$
	d\lambda_f(\xi,b)
	=
	e^{i\theta_f(\xi,b)}\,d|\lambda_f|(\xi,b).
	$$
	Since $f$ is real-valued ,
    %and the real-part map is bounded and linear from the complexification of $W^{m,p}(v_{-u};\rd)$ into	$W^{m,p}(v_{-u};\rd)$, 
    taking real parts in
	\eqref{eq:unbounded-Bochner-integral} gives
	$$
	f(x)
	=
	\int_{\rd\times\rr}
	\frac{v_s(b)}{v_{m+s}(\xi)}
	\sigma(\xi\cdot B^{-1}x+b)
	\cos\bigl(\pi x\cdot Mx-\theta_f(\xi,b)\bigr)
	\,d|\lambda_f|(\xi,b).
	$$
	% This is an intermediate representation. 
    To express it using a dictionary
	that is independent of the target function $f$, set
	$$
	\Theta
	:=
	\rd\times\rr\times[0,2\pi)
	$$
	and define, for $(\xi,b,\theta)\in\Theta$,
	$$
	\widetilde\varrho_S(x,\xi,b,\theta)
	:=
	\frac{v_s(b)}{v_{m+s}(\xi)}
	\sigma(\xi\cdot B^{-1}x+b)
	\cos\bigl(\pi x\cdot Mx-\theta\bigr).
	$$
	The corresponding real-valued weighted dictionary is
	$$
	\widetilde{\bb D}_S
	:=
	\left\{
	\widetilde\varrho_S(\cdot,\xi,b,\theta):
	(\xi,b,\theta)\in\Theta
	\right\}.
	$$
	In particular, both the atoms and the dictionary are independent of $f$.
	
	We now encode the target-dependent phase in the representing measure.
	Define a positive Borel measure $\mu_f$ on $\Theta$ by
	$$
	\mu_f(E)
	:=
	\int_{\rd\times\rr}
	\mathbf 1_E\bigl(\xi,b,\theta_f(\xi,b)\bigr)
	\,d|\lambda_f|(\xi,b)
	$$
	for every Borel set $E\subseteq\Theta$. Equivalently,
	$$
	d\mu_f(\xi,b,\theta)
	=
	d|\lambda_f|(\xi,b)\,
	\delta_{\theta_f(\xi,b)}(d\theta).
	$$
	The measurability of $\theta_f$ ensures that $\mu_f$ is well defined, and
	its total mass satisfies
	$$
	\mu_f(\Theta)
	=
	|\lambda_f|(\rd\times\rr)
	=
	\|\lambda_f\|.
	$$
	
	Moreover,
	$$
	\widetilde\varrho_S(\cdot,\xi,b,\theta)
	=
	\operatorname{Re}\left(
	e^{i\theta}\widetilde\rho_S(\cdot,\xi,b)
	\right).
	$$
	% Since weak differentiation commutes with the real-part map,
    The uniform
	bound for the complex-valued atoms gives
	$$
	\sup_{(\xi,b,\theta)\in\Theta}
	\|\widetilde\varrho_S(\cdot,\xi,b,\theta)\|
	_{W^{m,p}(v_{-u};\rd)}
	\le
	K_{\widetilde\rho_S}
	<
	\infty.
	$$
	The map
	$$
	\Theta\ni(\xi,b,\theta)
	\longmapsto
	\widetilde\varrho_S(\cdot,\xi,b,\theta)
	\in
	W^{m,p}(v_{-u};\rd)
	$$
	is strongly measurable. Indeed, the map
	$(\xi,b)\mapsto\widetilde\rho_S(\cdot,\xi,b)$ is strongly measurable by
	the argument above, and the map
	$$
	(g,\theta)
	\longmapsto
	\operatorname{Re}(e^{i\theta}g)
	$$
	is continuous from the Sobolev space times $[0,2\pi)$ into
	$W^{m,p}(v_{-u};\rd)$.
	
	Consequently,
	$$
	\int_\Theta
	\|\widetilde\varrho_S(\cdot,\xi,b,\theta)\|
	_{W^{m,p}(v_{-u};\rd)}
	\,d\mu_f(\xi,b,\theta)
	\le
	K_{\widetilde\rho_S}\mu_f(\Theta)
	=
	K_{\widetilde\rho_S}\|\lambda_f\|
	<
	\infty.
	$$
	Hence the following Bochner integral is well defined in
	$W^{m,p}(v_{-u};\rd)$, and the preceding real-part representation becomes
	$$
	f
	=
	\int_\Theta
	\widetilde\varrho_S(\cdot,\xi,b,\theta)
	\,d\mu_f(\xi,b,\theta).
	$$
	
	Applying Proposition \ref{prop:variation_norm} to this parameter-space
	representation, and then using
	\eqref{eq:unbounded-measure-bound}, yields
	\begin{equation}
		\begin{aligned}
			\|f\|_{\mathcal K(\widetilde{\bb D}_S)}
			&\le
			\mu_f(\Theta)
			=
			\|\lambda_f\|
			\le
			C\|f\|_{\mathscr B^S_{m+s}(\rd)}.
		\end{aligned}
		\label{eq:metaplectic-unbounded-variation-bound}
	\end{equation}
	
	\emph{Step 4: Maurey's estimate and the original dictionary.}
	The space $L^p(v_{-u};\rd)$ is
    a type $2$ for $p\ge2$.
    Hence
	$W^{m,p}(v_{-u};\rd)$ is also of type $2$; cf.\
	\cite[Proposition 21]{Abdeljawad24WeightedApproximationBarron}.
	Applying Proposition \ref{prop:approximation_type2} with
	$$
	M_f:=\|f\|_{\mathcal K(\widetilde{\bb D}_S)}
	$$
	and using \eqref{eq:metaplectic-unbounded-variation-bound}, we obtain
	$$
	\inf_{g_N\in\Sigma_{N,M_f}(\widetilde{\bb D}_S)}
	\|f-g_N\|_{W^{m,p}(v_{-u};\rd)}
	\le
	CN^{-1/2}\|f\|_{\mathscr B^S_{m+s}(\rd)}.
	$$
	Finally, define the atoms of the original dictionary by
	$$
	\varrho_S(x,\xi,b,\theta)
	:=
	\sigma(\xi\cdot B^{-1}x+b)
	\cos\bigl(\pi x\cdot Mx-\theta\bigr).
	$$
	Then
	$$
	\widetilde\varrho_S(x,\xi,b,\theta)
	=
	\frac{v_s(b)}{v_{m+s}(\xi)}
	\varrho_S(x,\xi,b,\theta).
	$$
	The prefactor is positive, so every element of
	$\Sigma_{N,M_f}(\widetilde{\bb D}_S)$ belongs to
	$\Sigma_N(\bb D_S)$ after the prefactor is absorbed into the scalar
	coefficient. Therefore,
	$$
	\inf_{f_N\in\Sigma_N(\bb D_S)}
	\|f-f_N\|_{W^{m,p}(v_{-u};\rd)}
	\le
	CN^{-1/2}\|f\|_{\mathscr B^S_{m+s}(\rd)}.
	$$
	All constants may depend only on $d,m,p,u,\nu,s,S$, and the fixed activation
	$\sigma$, and are independent of $f$ and $N$. This proves
	\eqref{eq:unbounded-main-rate}.
\end{proof}

	%%%%%%%%%%%%%%%%%%%%%%%%%%%%%%%%%%%%%%%%%%%%%%%%%%%%%%%%%%%%%%%%%%
	\section{Numerical Experiment: Harmonic-Oscillator Schrödinger Equation}
	\label{sec:numerical_schrodinger_pinn}
	%%%%%%%%%%%%%%%%%%%%%%%%%%%%%%%%%%%%%%%%%%%%%%%%%%%%%%%%%%%%%%%%%%
	
	\paragraph{Problem and objective.}
	Following the physics-informed framework of
	\cite{Raissi2019PhysicsinformedNeuralNetworks}, we consider
	\begin{equation}
		i\partial_t\psi
		+\frac{1}{2}\partial_{xx}\psi
		-\frac{1}{2}x^2\psi=0,
		\qquad
		(x,t)\in[-10,10]\times[0,1].
		\label{eq:schrodinger_experiment}
	\end{equation}
	For the harmonic-oscillator mode $n$, the reference solution is
\begin{equation}
\psi_n^\star(x,t)
=
\phi_n(x)e^{-i(n+1/2)t},
\qquad
\phi_n(x)
=
\frac{\pi^{-1/4}}{\sqrt{2^n n!}},
H_n(x)e^{-x^2/2},
\label{eq:schrodinger_reference}
\end{equation}
where $H_n$ is the Hermite polynomial in the physicists' convention, i.e.,
\begin{equation*}
H_n(x)
=
(-1)^n e^{x^2}
\frac{d^n}{dx^n}
e^{-x^2},
\qquad
n\in\mathbb Z_+
\end{equation*}
see, for instance,
\cite{griffiths2018quantum,teschl2014quantum,thangavelu1993hermite}.
	We impose \(\psi(x,0)=\phi_n(x)\) and use the exact traces of
	\(\psi_n^\star\) at \(x=\pm10\).
	
	Writing
$\psi_\theta=u_\theta+iv_\theta$,
where $u_\theta$ and $v_\theta$ are real-valued functions, we decompose
the Schrödinger equation into its real and imaginary parts. Indeed,
$$
i\partial_t\psi_\theta
=
i\partial_t\left(u_\theta+iv_\theta\right)
=
-\partial_t v_\theta
+i\partial_t u_\theta,
$$
while
$$
\frac{1}{2}\partial_{xx}\psi_\theta
-\frac{1}{2}x^2\psi_\theta
=
\left(
\frac{1}{2}\partial_{xx}u_\theta
-\frac{1}{2}x^2u_\theta
\right)
+
i\left(
\frac{1}{2}\partial_{xx}v_\theta
-\frac{1}{2}x^2v_\theta
\right).
$$
Therefore,
$$
i\partial_t\psi_\theta
+\frac{1}{2}\partial_{xx}\psi_\theta
-\frac{1}{2}x^2\psi_\theta
=
r_{\theta,u}
+i r_{\theta,v},
$$
where the real and imaginary residual components are given by
\begin{equation}
	r_{\theta,u}
	=
	-\partial_t v_\theta
	+\frac{1}{2}\partial_{xx}u_\theta
	-\frac{1}{2}x^2u_\theta,
	\qquad
	r_{\theta,v}
	=
	\partial_t u_\theta
	+\frac{1}{2}\partial_{xx}v_\theta
	-\frac{1}{2}x^2v_\theta.
\end{equation}

   The training objective is
\begin{equation}
	\mathcal L
	=
	0.1\,\mathcal L_{\rm PDE}
	+\mathcal L_{\rm IC}
	+0.1\,\mathcal L_{\rm BC}
	+\mathcal L_{\rm mass}
	+100\,\mathcal L_{\rm snap}.
	\label{eq:schrodinger_training_loss}
\end{equation}
Its components are defined by
\begin{align}
	\mathcal L_{\rm PDE}
	&=
	\frac{1}{N_{\rm PDE}}
	\sum_{j=1}^{N_{\rm PDE}}
	\left(
	r_{\theta,u}(x_j,t_j)^2
	+
	r_{\theta,v}(x_j,t_j)^2
	\right),
	\\
	\mathcal L_{\rm IC}
	&=
	\frac{1}{N_{\rm IC}}
	\sum_{j=1}^{N_{\rm IC}}
	\left|
	\psi_\theta(x_j,0)-\phi_n(x_j)
	\right|^2,
	\\
	\mathcal L_{\rm BC}
	&=
	\frac{1}{2N_{\rm BC}}
	\sum_{j=1}^{N_{\rm BC}}
	\sum_{x_b\in\{-10,10\}}
	\left|
	\psi_\theta(x_b,t_j)
	-\psi_n^\star(x_b,t_j)
	\right|^2,
	\\
	\mathcal L_{\rm mass}
	&=
	\frac{1}{N_{\rm mass}}
	\sum_{j=1}^{N_{\rm mass}}
	\left(
	M_\theta(t_j)-M_n^\star
	\right)^2,
	\\
	\mathcal L_{\rm snap}
	&=
	\frac{1}{4N_{\rm snap}}
	\sum_{\tau\in\mathcal T_{\rm snap}}
	\sum_{j=1}^{N_{\rm snap}}
	\left|
	\psi_\theta(x_j,\tau)
	-\psi_n^\star(x_j,\tau)
	\right|^2,
\end{align}
where
$$
M_\theta(t)
=
\int_{-10}^{10}
|\psi_\theta(x,t)|^2\,dx,
\qquad
M_n^\star
=
\int_{-10}^{10}
|\phi_n(x)|^2\,dx,
$$
and
$$
\mathcal T_{\rm snap}
=
\{0.25,0.50,0.75,1.00\}.
$$
The mass integrals are approximated numerically by the trapezoidal rule.
Thus, the experiment is a data-assisted PINN comparison rather than a
residual-only PINN experiment.
	\paragraph{Architectures.}
The inputs are normalized to $[-1,1]^2$. The plain model is a fully
connected tanh network with six affine layers and architecture
$$
2\longrightarrow128\longrightarrow128\longrightarrow128
\longrightarrow128\longrightarrow128\longrightarrow2.
$$
The metaplectic-inspired model has the same number of affine layers and
architecture
$$
2\longrightarrow125\longrightarrow128\longrightarrow128
\longrightarrow128\longrightarrow126\longrightarrow2.
$$
These widths are chosen so that both models have exactly $66{,}690$
trainable parameters.

Let $z\in[-1,1]^2$ denote the normalized space-time input and let
$h_0=z$. The first metaplectic hidden layer uses the fixed invertible
matrix $B$ through
$$
a_1=W_1B^{-1}z+b_1,
$$
whereas the subsequent affine pre-activations are
$$
a_\ell=W_\ell h_{\ell-1}+b_\ell,
\qquad
\ell=2,\ldots,5.
$$
The hidden states are defined componentwise by
$$
h_{\ell,j}
\sigma(a_{\ell,j})
\cos\left(
\pi z^TM_\ell z-\theta_{\ell,j}
\right),
$$
where
$$
\sigma(a)=\tanh(a+1)-\tanh(a).
$$
Here, $M_\ell\in\mathbb R^{2\times2}$ is a trainable matrix that may vary
from one hidden layer to another, and $\theta_{\ell,j}\in[0,2\pi)$ is a
trainable phase associated with the $j$-th neuron of layer $\ell$.
% No symmetry constraint is imposed on $M_\ell$ during training.
The final
output is obtained through the affine map
$$
\begin{pmatrix}
u_\theta(x,t)\\
v_\theta(x,t)
\end{pmatrix}
=
W_6h_5(z)+b_6,
$$
This architecture is motivated by the metaplectic dictionary appearing
in \cref{thm:MC_metaplectic_Wmr_noeps}, but allows the quadratic matrix
and phase parameters to vary across layers. Moreover, the activation
$\sigma(a)=\tanh(a+1)-\tanh(a)$ and all its derivatives decay
exponentially, and therefore satisfy the weighted regularity assumptions
used in the theorem, see 
\cite{Siegel20ApproximationRatesNeural}.

	\paragraph{Training and reporting.}
	We consider
	\[
	n\in\{0,1,\ldots,15,20\}
	\]
	and five matched random seeds for each architecture. Every run uses
	\(100{,}000\) epochs with two AdamW updates per epoch, learning rate
	\(10^{-3}\), and weight decay \(10^{-5}\). At each update, we draw
	\(2048\) interior points, \(512\) initial samples, \(512\) boundary
	times at each endpoint, \(16\) mass-evaluation times, and \(256\)
	spatial points at each snapshot time. The mass term uses a
	\(128\)-point trapezoidal rule, while its reference value is computed
	on the fixed \(512\)-point initial grid. Random training points are
	resampled at every update. We use gradient clipping at norm \(1\) and
	reduce the learning rate by a factor \(0.9\) when the monitoring error
	plateaus.
	
	Let \(\mathcal L_{n,s}^{(k)}\) denote the epoch loss for mode \(n\),
	seed \(s\), and epoch \(k\). The terminal score of a run is defined by
	\begin{equation}
		\overline{\mathcal L}_{n,s}
		=
		\operatorname{median}
		\left\{
		\mathcal L_{n,s}^{(k)}
		:
		k>0.95K
		\right\},
		\label{eq:final_window_loss}
	\end{equation}
	where \(K=100{,}000\). Convergence curves are reported as the median
	over seeds with the \(25\)--\(75\%\) band. To compare the two
	architectures throughout training, we also use the matched-seed
	advantage
	\begin{equation}
		A_n^{(k)}
		=
		\operatorname{median}_{s}
		\log_{10}
		\left(
		\frac{
			\mathcal L_{n,s,\mathrm{plain}}^{(k)}
		}{
			\mathcal L_{n,s,\mathrm{meta}}^{(k)}
		}
		\right).
		\label{eq:matched_seed_advantage}
	\end{equation}
	Positive values of \(A_n^{(k)}\) indicate a lower objective for the
	metaplectic-inspired network.

	\paragraph{Results.}
	Figure~\ref{fig:schrodinger_final_loss} shows that the terminal
	training objective increases with the mode index for both
	architectures, reflecting the increasing spatial complexity of the
	Hermite eigenfunctions. The metaplectic-inspired network attains a
	lower median final-window loss at the displayed modes, and the
	difference generally becomes more pronounced as \(n\) increases. In
	particular, at \(n=20\), the separation is approximately one order of
	magnitude.
	
	The matched-seed comparison in
	Figure~\ref{fig:schrodinger_advantage} reveals a distinction between
	early- and late-stage optimization. The plain network frequently has
	the smaller objective during the initial training phase, whereas the
	metaplectic-inspired model becomes favorable later, especially for
	the more eigenmodes. The individual convergence curves in
	Figure~\ref{fig:schrodinger_training_curves} confirm that the observed
	advantage is primarily associated with a lower late-training plateau,
	rather than uniformly faster initial convergence.

	\begin{figure}[H]
		\centering
		\includegraphics[width=0.5\textwidth]{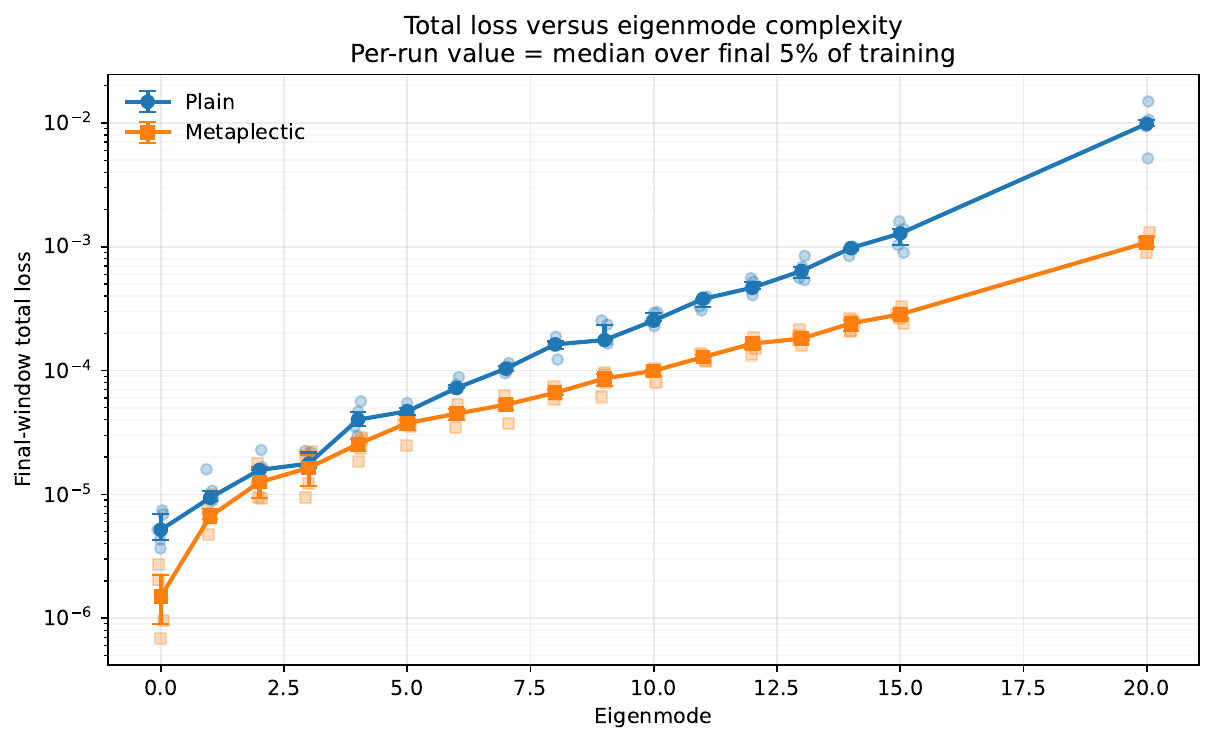}
		\caption{
			Terminal training objective as a function of the harmonic-oscillator
			mode \(n\). For each run, the reported value is the median total loss
			over the final \(5\%\) of training. The curves and error bars show the
			median and interquartile range over five matched seeds, and the
			transparent points show the individual runs.
		}
		\label{fig:schrodinger_final_loss}
	\end{figure}

	\begin{figure}[H]
		\centering
		\includegraphics[width=0.5\textwidth]{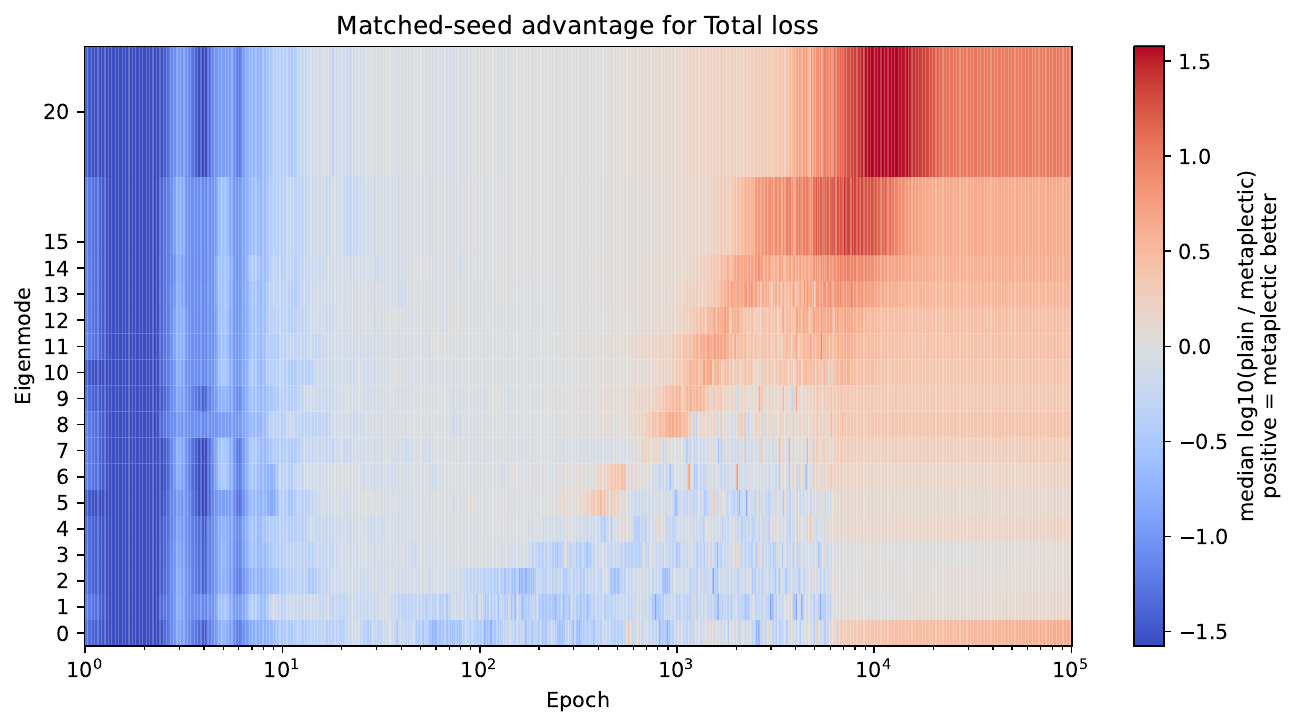}
		\caption{
			Matched-seed advantage during training. The color at mode \(n\) and
			epoch \(k\) is
			\(\operatorname{median}_s
			\log_{10}(
			\mathcal L_{\mathrm{plain}}/
			\mathcal L_{\mathrm{meta}}
			)\).
			Positive values indicate a smaller total loss for the
			metaplectic-inspired model, while negative values indicate a smaller
			loss for the plain model.
		}
		\label{fig:schrodinger_advantage}
	\end{figure}

	\begin{figure}[H]
		\centering
		\includegraphics[width=0.99\textwidth]{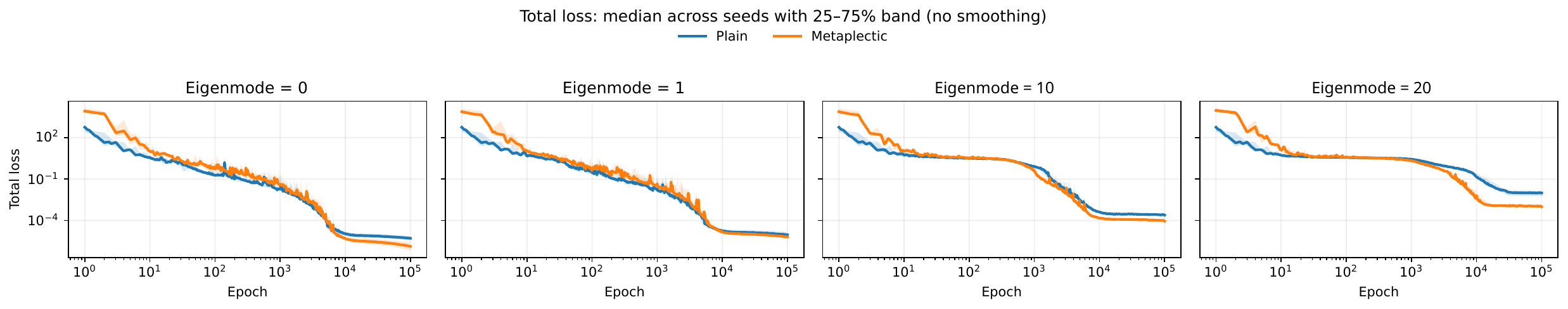}
		\caption{
			Training objective for each harmonic-oscillator mode. Solid curves
			show the median over five seeds, and shaded regions show the
			\(25\)--\(75\%\) range. No temporal smoothing is applied.
		}
		\label{fig:schrodinger_training_curves}
	\end{figure}

	%%%%%%%%%%%%%%%%%%%%%%%%%%%%%%%%%%%%%%%%%%%%%%%%%%%%%%%%%%%%%%%%%% 
	
	\section*{Acknowledgments}
	%%%%%%%%%%%%%%%%%%%%%%%%%%%%%%%%%%%%%%%%%%%%%%%%%%%%%%%%%%%%%%%%%% 
	A.~Abdeljawad acknowledges the support of the Austrian Science Fund (FWF) through project PAT4788625 (Grant-DOI: 10.55776/PAT4788625). M. Carioni acknowledges the support of NWO through the Vidi grant \emph{SPARGO: Exploring and Exploiting the Geometric Landscape of Infinite-Dimensional Sparse Optimization} (Grant Number VI.Vidi.243.200). E. Cordero has been supported by the Gruppo Nazionale per l’Analisi Matematica, la Probabilità e le loro Applicazioni (GNAMPA) of the Istituto Nazionale di Alta Matematica (INdAM).

	%%%%%%%%%%%%%%%%%%%%%%%%%%%%%%%%%%%%%%%%%%%%%%%%%%%%%%%%%%%%%%%%%%
	\appendix
	%%%%%%%%%%%%%%%%%%%%%%%%%%%%%%%%%%%%%%%%%%%%%%%%%%%%%%%%%%%%%%%%%%

	\section{Additional Results}

    We start by proving uniform Sobolev bounds for a weighted dictionary in bounded domains.

	\begin{lemma}[Uniform Sobolev bound for the bounded-domain dictionary]
		\label{lem:dictionary-uniform-bound}
		Consider \(n\in\zzp{}\), \(1\le r\le\infty\), and 
		\(\Omega\subset\rd\) an open, bounded set. Let
		\(S\in\mathrm{Sp}(d,\rr)\) be free, written as in \eqref{blockS}. Assume that
		\(	\sigma\in W^{n,\infty}(v_s;\rr),\)
		for some $s>1$.
		Set 
		\(
		R_{\Omega,B}
		:=
		\sup_{x\in\Omega}\abs{B^{-1}x},
		\)
		and, for \((\xi,b)\in\rd\times\rr{}\), define the weighted atoms
		\[
		\widetilde{\rho}_S(x,\xi,b)
		:=
		\frac{\sigma\left(
			\xi\cdot B^{-1}x+b
			\right)
			e^{-\pi i\,x\cdot(B^{-1}A)x}}{\omega(\xi,b)},
		\]
		with the weight defined as 
		\begin{equation}\label{eq:weightomega}
			\omega(\xi,b)
			:=
			\frac{v_n(\xi)}
			{
				v_s\left(
				\left(\abs{b}-R_{\Omega,B}\abs{\xi}\right)_+
				\right)
			}.
		\end{equation}
		Then the family
		\[
		\widetilde{\bb D}_S
		:=
		\left\{
		\widetilde{\rho}_S(\cdot,\xi,b)
		:
		(\xi,b)\in\rd\times\rr{}
		\right\}
		\]
		is uniformly bounded in \(W^{n,r}(\Omega)\).
	\end{lemma}
	
	\begin{proof}
		We show that, for every multi-index \(\alpha\in \zzp{d}\) with \(|\alpha|\le n\), there exists a constant
		\(C>0\), independent of \(\xi\) and \(b\), such that
		\[
		\|\partial^\alpha \widetilde{\rho}_S(\cdot,\xi,b)\|_{L^r(\Omega)}
		\le C .
		\]
		Since the weight \(\omega\) is independent of \(x\) and \(\sigma\in W^{n,\infty}(v_s;\rr)\), we may differentiate
		\(\widetilde{\rho}_S\) with respect to \(x\) up to order \(n\).
		Thus, we obtain
		\[
		\|\partial^\alpha \widetilde{\rho}_S(\cdot,\xi,b)\|_{L^r(\Omega)}
		=
		\frac{1}{\omega(\xi,b)}
		\left\|
		\partial^\alpha\!\left(
		\sigma\!\left(\xi\cdot B^{-1}x+b\right)
		e^{-\pi i\,x\cdot (B^{-1}A)x}
		\right)
		\right\|_{L^r(\Omega)} .
		\]
		Set
		\[
		w_B(\xi):=(B^{-1})^T\xi,
		\qquad
		Q:=\pi i\bigl(B^{-1}A+(B^{-1}A)^T\bigr).
		\]
		By Leibniz's formula and the chain rule,
		\begin{multline*}
			\partial^\alpha\!\left(
			\sigma\!\left(\xi\cdot B^{-1}x+b\right)
			e^{-\pi i\,x\cdot (B^{-1}A)x}
			\right)
			\\
			=
			e^{-\frac12 x\cdot Qx}
			\sum_{\beta\le \alpha}\binom{\alpha}{\beta}
			\bigl(w_B(\xi)\bigr)^\beta
			\sigma^{(|\beta|)}\!\bigl(w_B(\xi)\cdot x+b\bigr)
			(-1)^{|\alpha-\beta|}
			H_{\alpha-\beta}(x,Q),
		\end{multline*}
		where, for \(\gamma\in\zzp{d}\),
		\[
		H_\gamma(x,Q)
		:=
		e^{\frac12 x\cdot Qx}(-\partial_x)^\gamma e^{-\frac12 x\cdot Qx}.
		\]
		Interested readers on generalized Hermite polynomials can check, e.g. 
		\cite[Property 4.13]{Terdik21MultivariateStatisticalMethods} and
		\cite[Section 8.3]{Dunkl14OrthogonalPolynomialsSeveral}.
		Hence,
		\begin{align*}
			\|\partial^\alpha \widetilde{\rho}_S(\cdot,\xi,b)\|_{L^r(\Omega)}
			&\le
			\sum_{\beta\le \alpha}\binom{\alpha}{\beta}
			\frac{|w_B(\xi)|^{|\beta|}}{\omega(\xi,b)}
			\left\|
			\sigma^{(|\beta|)}\!\bigl(w_B(\xi)\cdot x+b\bigr)
			e^{-\frac12 x\cdot Qx}
			H_{\alpha-\beta}(x,Q)
			\right\|_{L^r(\Omega)} .
		\end{align*}
		Since \(\sigma\in W^{n,\infty}(v_s;\rr)\),
		it follows that
		\begin{align*}
			\|\partial^\alpha \widetilde{\rho}_S(\cdot,\xi,b)\|_{L^r(\Omega)}
			&\le
			\|\sigma\|_{W^{n,\infty}(v_s;\rr)}
			\sum_{\beta\le \alpha}\binom{\alpha}{\beta}
			\frac{|w_B(\xi)|^{|\beta|}}{\omega(\xi,b)}
			\left\|
			\frac{e^{-\frac12 x\cdot Qx}H_{\alpha-\beta}(x,Q)}
			{v_s\!\bigl(|w_B(\xi)\cdot x+b|\bigr)}
			\right\|_{L^r(\Omega)} .
		\end{align*}
		Now, for \(x\in\Omega\),
		\[
		|w_B(\xi)\cdot x+b|
		=
		\left|\xi\cdot B^{-1}x +b\right|
		\ge
		\Bigl(|b|-R_{\Omega,B}|\xi|\Bigr)_+.
		\]
		Since \(v_s\) is increasing on \([0,\infty)\), it follows that
		\[
		v_s\!\bigl(|w_B(\xi)\cdot x+b|\bigr)^{-1}
		\le
		v_s\!\left(\Bigl(|b|-R_{\Omega,B}|\xi|\Bigr)_+\right)^{-1}.
		\]
		Hence,
		\begin{multline*}
			\|\partial^\alpha \widetilde{\rho}_S(\cdot,\xi,b)\|_{L^r(\Omega)}
			\\
			\le
			\|\sigma\|_{W^{n,\infty}(v_s;\rr)}
			\sum_{\beta\le \alpha}\binom{\alpha}{\beta}
			\frac{|w_B(\xi)|^{|\beta|}}
			{\omega(\xi,b)\,
				v_s\!\left(\bigl(|b|-R_{\Omega,B}\bigl|\xi\bigr|\bigr)_+\right)}
			\left\|
			e^{-\frac12 x\cdot Qx}H_{\alpha-\beta}(x,Q)
			\right\|_{L^r(\Omega)} .
		\end{multline*}
		The fact that \(Q\) is purely imaginary implies that \(|e^{-\frac12 x\cdot Qx}|=1\). Moreover,
		\(H_{\alpha-\beta}(\cdot,Q)\) is a polynomial and \(\Omega\) is bounded. Thus
		\(H_{\alpha-\beta}(\cdot,Q)\in L^r(\Omega)\), and there exists
		\(C_{\alpha-\beta,d,Q}>0\) such that
		\[
		\|e^{-\frac12 x\cdot Qx}H_{\alpha-\beta}(x,Q)\|_{L^r(\Omega)}
		=
		\|H_{\alpha-\beta}(\cdot,Q)\|_{L^r(\Omega)}
		\le
		C_{\alpha-\beta,d,Q}.
		\]
		Consequently, we get
		\begin{align*}
			\|\partial^\alpha \widetilde{\rho}_S(\cdot,\xi,b)\|_{L^r(\Omega)}
			&\le
			\|\sigma\|_{W^{n,\infty}(v_s;\rr)}
			\sum_{\beta\le \alpha}\binom{\alpha}{\beta}
			C_{\alpha-\beta,d,Q}
			\frac{|w_B(\xi)|^{|\beta|}}
			{\omega(\xi,b)\,
				v_s\!\left(\bigl(|b|-R_{\Omega,B}\bigl|\xi\bigr|\bigr)_+\right)}.
		\end{align*}
		Using the definition of \(\omega(\xi,b)\) in \cref{eq:weightomega},
		it is straightforward that
		\[
		\frac{1}
		{\omega(\xi,b)\,
			v_s\!\left(\bigl(|b|-R_{\Omega,B}\bigl|\xi\bigr|\bigr)_+\right)}
		=
		\frac{1}{v_n(\xi)}.
		\]
		Therefore,
		\begin{align*}
			\|\partial^\alpha \widetilde{\rho}_S(\cdot,\xi,b)\|_{L^r(\Omega)}
			&\le
			\|\sigma\|_{W^{n,\infty}(v_s;\rr)}
			\sum_{\beta\le \alpha}\binom{\alpha}{\beta}
			C_{\alpha-\beta,d,Q}
			\frac{|w_B(\xi)|^{|\beta|}}{v_n(\xi)}.
		\end{align*}
		Finally, since
		\[
		|w_B(\xi)|
		=
		\left|(B^{-1})^T\xi\right|
		\le
		\|B^{-1}\|\,|\xi|,
		\]
		we obtain
		\[
		\frac{|w_B(\xi)|^{|\beta|}}{v_n(\xi)}
		\le
		\left(\|B^{-1}\|\right)^{|\beta|}
		\langle \xi\rangle^{|\beta|-n}.
		\]
		Since \(|\beta|\le |\alpha|\le n\), the factor \(\langle \xi\rangle^{|\beta|-n}\) is uniformly bounded by \(1\). Thus
		\[
		\|\partial^\alpha \widetilde{\rho}_S(\cdot,\xi,b)\|_{L^r(\Omega)}
		\le C,
		\qquad
		\text{for all }\xi\in\rd,\ b\in\rr,
		\]
		for some constant \(C>0\) independent of \(\xi\) and \(b\). Summing over all \(|\alpha|\le n\) gives
		\[
		\|\widetilde{\rho}_S(\cdot,\xi,b)\|_{W^{n,r}(\Omega)}
		\le C,
		\qquad
		\text{for all }\xi\in\rd,\ b\in\rr,
		\]
		for some constant \(C=C(n, d, Q)>0\).
	\end{proof}
	
	\begin{lemma}[One-dimensional weighted estimate]\label{lem:weight-estimate}
		Let \(1\le p<\infty\), let \(s>1\), and let
		\(q_1>s+1/p\). For \(\lambda\ge0\) and \(b\in\rr{}\), set
		\[
		I_\lambda(b)
		:=
		\int_{\rr}
		v_{-sp}(\lambda t+b)v_{-q_1p}(t)\,dt .
		\]
		Define
		\[
		\theta(\lambda):=
		\begin{cases}
			1, & 0\le \lambda\le 1,\\
			\lambda^{-1}, & \lambda\ge 1.
		\end{cases}
		\]
		Then there exists a constant \(C>0\), independent of \(\lambda\) and
		\(b\), such that
		\[
		I_\lambda(b)
		\le
		C v_{-sp}\bigl(\theta(\lambda)|b|\bigr),
		\qquad \lambda\ge0,\ b\in\rr{}.
		\]
	\end{lemma}
	
	\begin{proof}
		We distinguish the cases \(0\le\lambda\le1\) and \(\lambda\ge1\).
		
		First, let \(0\le\lambda\le1\). If \(\lambda=0\), then
		\[
		I_0(b)
		=
		v_{-sp}(b)
		\int_{\rr}v_{-q_1p}(t)\,dt
		\le
		C v_{-sp}(b),
		\]
		because \(q_1p>1\). Throughout the proof, \(C>0\) denotes a constant independent of \(\lambda\), \(b\), and \(t\), whose value may change from line to line.
		
		If \(0<\lambda\le1\), then submultiplicativity of polynomial weights gives
		\[
		v_s(b)
		\le
		v_s(\lambda t+b)v_s(\lambda t).
		\]
		Hence
		\[
		v_{-sp}(\lambda t+b)
		\le
		v_{sp}(\lambda t)v_{-sp}(b)
		\le
		v_{sp}(t)v_{-sp}(b),
		\]
		where we used the fact that \(v_{sp}(\lambda t)\le v_{sp}(t)\) when  \(0<\lambda\le1\). Therefore
		\[
		\begin{aligned}
			I_\lambda(b)
			&\le
			v_{-sp}(b)
			\int_{\rr}v_{sp}(t)v_{-q_1p}(t)\,dt \\
			&=
			v_{-sp}(b)
			\int_{\rr}v_{-(q_1-s)p}(t)\,dt
			\le
			C v_{-sp}(b),
		\end{aligned}
		\]
		because \((q_1-s)p>1\). Thus the desired estimate holds for
		\(0\le\lambda\le1\).
		
		Now let \(\lambda\ge1\). Since
		\(
		\lambda t+b
		=
		\lambda\left(t+\frac b\lambda\right),
		\)
		we have
		\[
		v_{-sp}(\lambda t+b)
		\le
		v_{-sp}\left(t+\frac b\lambda\right).
		\]
		Therefore, with \(\beta:=b/\lambda\),
		\[
		I_\lambda(b)
		\le
		\int_{\rr}
		v_{-sp}(t+\beta)v_{-q_1p}(t)\,dt .
		\]
		First, we split the full domain into the two regions
		\[
		E_1:=\{t\in\rr{}: |t|\le |\beta|/2\},
		\qquad
		E_2:=\{t\in\rr{}: |t|>|\beta|/2\}.
		\]
		On \(E_1\), we have \(|t+\beta|\ge |\beta| - |t| \ge  |\beta|/2\), and hence
		\[
		v_{-sp}(t+\beta)\le C v_{-sp}(|\beta|).
		\]
		Thus
		\[
		\int_{E_1}v_{-sp}(t+\beta)v_{-q_1p}(t)\,dt
		\le
		C v_{-sp}(|\beta|)
		\int_{\rr}v_{-q_1p}(t)\,dt
		\le
		C v_{-sp}(|\beta|),
		\]
		because \(q_1p>1\).
		
		On \(E_2\), we have \(|t|>|\beta|/2\), so
		\[
		v_{-q_1p}(t)
		\le
		C v_{-q_1p}(|\beta|)
		\le
		C v_{-sp}(|\beta|),
		\]
		since \(q_1>s\). Therefore
		\[
		\begin{aligned}
			\int_{E_2}v_{-sp}(t+\beta)v_{-q_1p}(t)\,dt
			&\le
			C v_{-sp}(|\beta|)
			\int_{\rr}v_{-sp}(t+\beta)\,dt 
			\le
			C v_{-sp}(|\beta|),
		\end{aligned}
		\]
		because \(sp>1\). Combining the estimates on \(E_1\) and \(E_2\) gives
		\[
		I_\lambda(b)
		\le
		C v_{-sp}\left(\frac{|b|}{\lambda}\right),
		\qquad \lambda\ge1.
		\]
		Together with the case \(0\le\lambda\le1\), this proves
		\[
		I_\lambda(b)
		\le
		C v_{-sp}\bigl(\theta(\lambda)|b|\bigr),
		\qquad \lambda\ge0.
		\]
		% The proof is complete.
	\end{proof}

	%%%%%%%%%%%%%%%%%%%%%%%%%%%%%%%%%%%%%%%%%%%%%%%%%%%%%%%%%%%%%%%%%%
	\printbibliography
	%%%%%%%%%%%%%%%%%%%%%%%%%%%%%%%%%%%%%%%%%%%%%%%%%%%%%%%%%%%%%%%%%%
	
\end{document}